\documentclass{article}

\usepackage{microtype}
\usepackage{graphicx}
\usepackage{subcaption}
\usepackage{booktabs} 
\usepackage{tabularx}
\usepackage{enumitem}

\usepackage{xurl} 
\usepackage{hyperref}

\usepackage[accepted]{icml2026}

\usepackage{amsmath}
\usepackage{amssymb}
\usepackage{mathtools}
\usepackage{amsthm}
\usepackage{longtable}
\usepackage{booktabs}
\usepackage{multirow}
\usepackage[most]{tcolorbox}

\usepackage[capitalize, noabbrev]{cleveref}

\theoremstyle{plain}
\newtheorem{theorem}{Theorem}[section]
\newtheorem{proposition}[theorem]{Proposition}
\newtheorem{lemma}[theorem]{Lemma}

\theoremstyle{definition}

\theoremstyle{remark}

\usepackage[disable, textsize=tiny]{todonotes}

\icmltitlerunning{GEM: Geometric Entropy Mixing for Optimal LLM Data Curation}

\begin{document}
\twocolumn[ \icmltitle{GEM: Geometric Entropy Mixing for Optimal LLM Data Curation}



    \icmlsetsymbol{equal}{*}

    \begin{icmlauthorlist}
        \icmlauthor{Yue Min}{equal,wizard,hkust} \icmlauthor{Ziyun Qiao}{equal,wizard,pku}
        \icmlauthor{Ruining Chen}{ustc} \icmlauthor{Yujun Li}{wizard}
    \end{icmlauthorlist}

    \icmlaffiliation{wizard}{Wizard Quant, Beijing, China} \icmlaffiliation{hkust}{The Hong Kong University of Science and Technology, Hong Kong SAR, China}
    \icmlaffiliation{pku}{Peking University, Beijing, China} \icmlaffiliation{ustc}{University of Science and Technology of China, Hefei, China}

    \icmlcorrespondingauthor{Yue Min}{minyue@wizardquant.com} \icmlcorrespondingauthor{Yujun Li}{liyujun@wizardquant.com}

    \icmlkeywords{Pretrain, Data mixing, Data curation}

    \vskip 0.3in ]



\printAffiliationsAndNotice{\icmlEqualContribution}

\begin{abstract}
    LLM pre-training efficacy increasingly depends on data composition rather
    than sheer volume. Yet, optimal mixing is hindered by categorization flaws:
    human taxonomies suffer from ontological misalignment, and Euclidean clustering
    fails to address embedding anisotropy. We introduce \textbf{GEM} (\textbf{G}eometric
    \textbf{E}ntropy \textbf{M}ixing), a framework reformulating data
    curation as a variational problem on the hypersphere augmented with a
    \textbf{mixing-balance regularizer}. By decoupling the generative prior
    and optimizing the objective via a provable \textbf{MM (Minorize-Maximize)}
    algorithm, GEM effectively counteracts the cluster collapse to discover
    balanced semantic structures invisible to Euclidean heuristics. We employ
    teacher-student distillation to scale this geometric fidelity to web-scale
    corpora and introduce the \textbf{Geometric Influence Score (GIS)} for
    interpretable taxonomy generation. Experiments with 1.1B-parameter
    models demonstrate that GEM establishes a new state-of-the-art when
    integrated into mixing strategies like DoReMi and RegMix, improving average
    downstream accuracy by up to \textbf{1.2\%} and offering a robust
    coordinate system for predictable data mixing.
\end{abstract}

\section{Introduction}

Data curation has emerged as the decisive factor in the performance of Large
Language Models (LLMs)~\cite{hoffmann2022training,gunasekar2023textbooks,penedo2023refinedweb},
shifting the research frontier from sheer parameter scaling to the strategic
"mixing" of heterogeneous data sources. As scaling laws~\cite{kaplan2020scaling}
evolve, the core challenge lies in partitioning massive-scale, unstructured
corpora into semantically distinct and balanced clusters, which is a prerequisite
for any principled data mixing strategy~\cite{ye2024data}. However, contemporary
approaches to data classification generally fall into two categories, both
of which face fundamental theoretical and practical bottlenecks.

\begin{figure}[t]
    \centering
    \includegraphics[width=\linewidth]{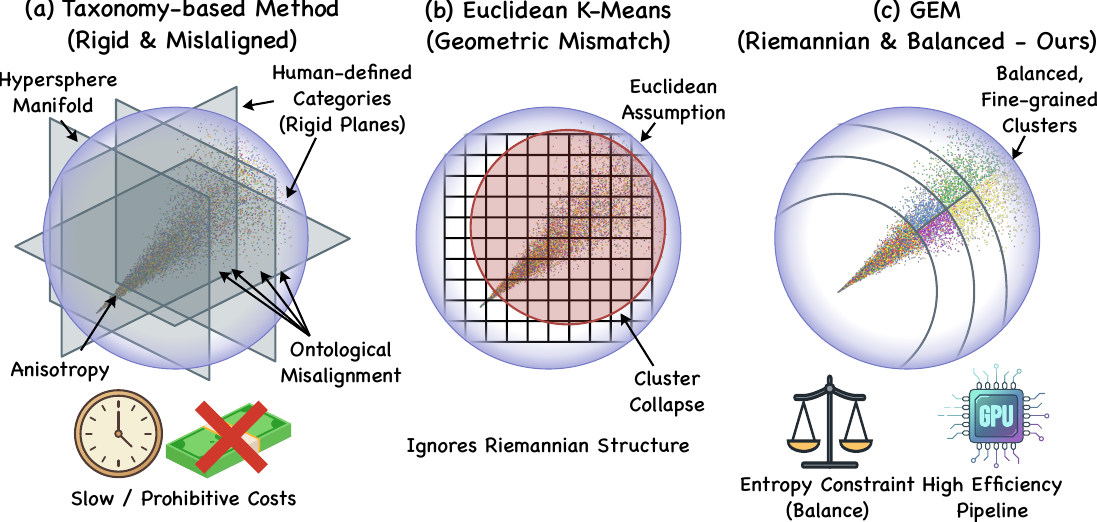}
    \caption{Schematic illustration of geometric mismatch in semantic
        clustering. While \textbf{(a)} taxonomy-based approaches are hindered by
        rigid misalignment and high costs, and \textbf{(b)} Euclidean clustering
        fails to handle embedding anisotropy leading to \textit{cluster collapse},
        \textbf{(c)} our proposed GEM framework utilizes MM-based inference to generate
        balanced, semantically distinct partitions on the hypersphere with
        superior efficiency.}
    \label{fig:comparison}
\end{figure}

The first category, taxonomy-based methods, relies on human-defined categorical
hierarchies~\cite{gpt3,llama2}. These approaches typically utilize high-capacity
LLMs or ensembles to assign labels to documents. However, as illustrated in
Figure~\ref{fig:comparison}(a), this paradigm suffers from a critical
ontological misalignment: human-centric categories often do not reflect the latent
semantic granularity required for self-supervised learning. Empirical
evidence suggests that even state-of-the-art models exhibit low inter-annotator
consistency when classifying nuanced web data, indicating that human taxonomies
fail to capture the true underlying distribution of model-relevant knowledge~\cite{maini2024rephrasing,semdedup}.
Furthermore, the cost of labeling renders this approach unsustainable,
especially given the dynamic nature of model development where data is
continuously updated, making the constant re-annotation of the corpus operationally
prohibitive.

Alternatively, unsupervised approaches like $K$-Means~\cite{kmeans} provide
a scalable option but are predicated on Euclidean geometry. This creates a fundamental
mismatch with modern neural embeddings (e.g., BGE~\cite{xiao2024c}, RoBERTa~\cite{liu2019roberta}),
which inherently reside on a high-dimensional hyperspherical manifold optimized
for cosine similarity. This geometric discrepancy is exacerbated by \textit{anisotropy},
the so-called ``cone effect''~\cite{li2020sentence}, where representations concentrate
in narrow, non-uniform sub-regions. Consequently, applying Euclidean
clustering to this Riemannian space precipitates ``cluster collapse,'' as illustrated
in Figure~\ref{fig:comparison}(b), where dominant clusters swallow the semantic
long-tail, severely limiting the diversity requisite for model
generalization~\cite{ethayarajh2019contextual,gao2021simcse}.

To bridge this gap, we introduce \textbf{GEM (Geometric Entropy Mixing)}, which
aligns semantic partitioning with the intrinsic Riemannian geometry of neural
representations. As conceptualized in Figure~\ref{fig:comparison}(c), GEM
departs from Euclidean heuristics by formulating the clustering task as an
entropy-regularized variational objective augmented with a mixing-balance
regularizer on the unit hypersphere. By explicitly decoupling the generative
prior and integrating balance regularizer on empirical mass into the von Mises-Fisher
Mixture Model (vMFMM), our method effectively mitigates embedding anisotropy
and prevents cluster collapse. This allows GEM to discover fine-grained semantic
structures and long-tail distributions that remain latent to traditional distance-based
methods, providing a more expressive semantic basis for data mixing. From a systems
perspective, GEM is architected for web-scale deployment through a Teacher-Student
distillation pipeline that achieves linear time complexity with respect to
corpus size. Furthermore, to bridge the gap between geometric clustering and
human-centric data curation, we introduce a Geometric Influence Score (GIS)-based
sampling method to generate an interpretable, fine-grained taxonomy with
descriptions for each semantic category. Extensive experiments demonstrate that
the data mixtures derived via GEM consistently yield superior scaling laws,
manifesting in lower validation perplexity and enhanced performance across
diverse downstream benchmarks compared to competitive baselines.

Our primary contributions are summarized as follows:
\begin{itemize}[topsep=0pt,itemsep=0pt,parsep=0pt,partopsep=0pt]
    \item \textbf{Geometric formulation with balance regularization.} We
          propose a hyperspherical variational framework with a novel mixing-balance
          regularizer to effectively prevent cluster collapse under embedding
          anisotropy.

    \item \textbf{Provable MM-based inference algorithm.} We derive a
          provable MM (Minorize-Maximize) algorithm that guarantees monotonic
          ascent, ensuring stable convergence for the regularized objective.

    \item \textbf{Scalable deployment with interpretability.} We enable
          linear-time inference via teacher--student distillation and introduce
          the Geometric Influence Score (GIS) for interpretable taxonomy generation.

    \item \textbf{Consistent gains in data mixing.} Experiments with 1.1B models
          demonstrate consistent performance gains over strong baselines across
          diverse benchmarks.
\end{itemize}

\begin{figure*}[t]
    \centering
    \includegraphics[width=0.95\linewidth]{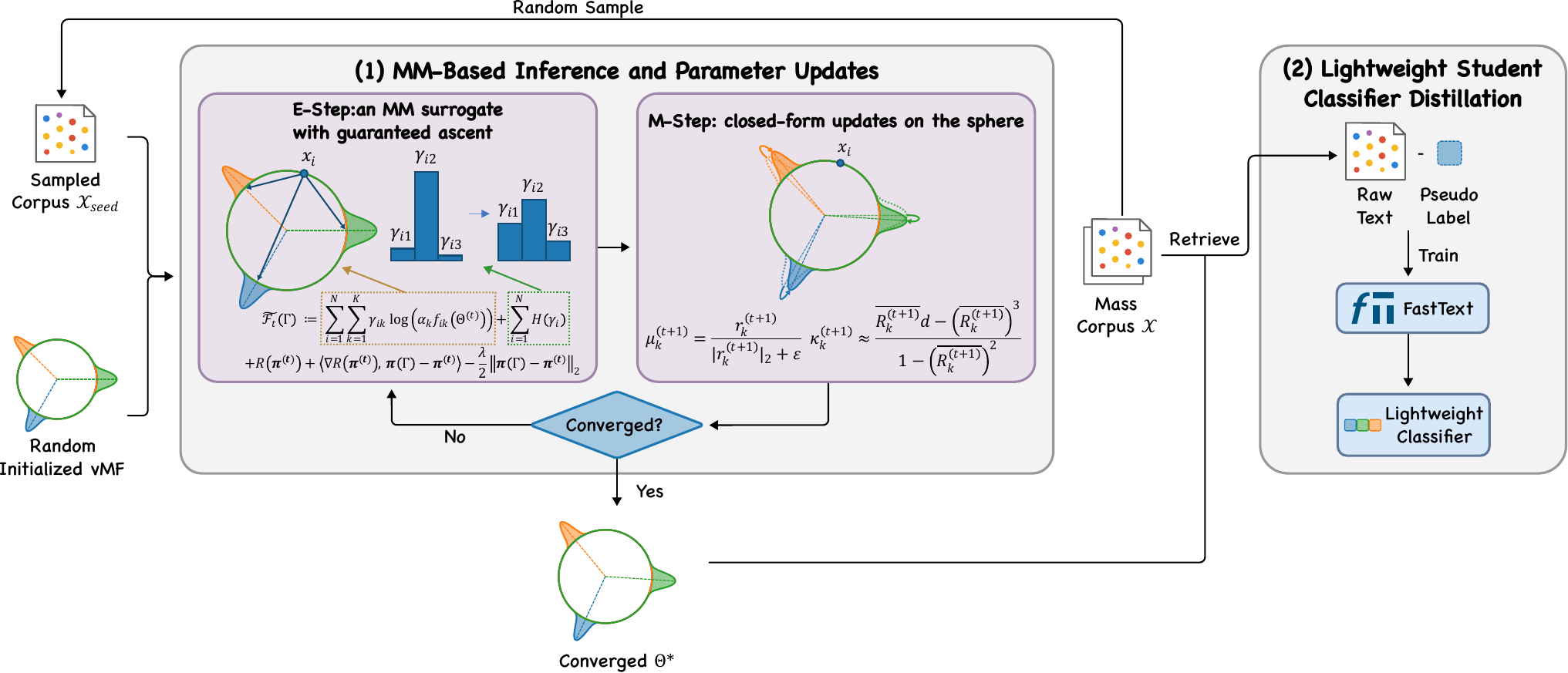}
    \caption{Schematic overview of the GEM framework. The pipeline consists of
        two phases: \textbf{(1) Geometric Optimization (Teacher):} We perform entropy-regularized
        clustering on the hypersphere using a Mixture of von Mises-Fisher (vMF) distributions.
        An Minorize–Maximize (MM) algorithm iteratively updates the Riemannian
        parameters ($\mu, \kappa$) on a seed corpus $\mathcal{X}_{seed}$ to
        discover semantic structures. \textbf{(2) Scalable Distillation (Student):}
        The converged geometric partitions are used to pseudo-label the mass
        corpus $\mathcal{X}$, guided by GIS score. These labels are then
        distilled into a lightweight FastText classifier, enabling efficient inference
        at scale.}
    \label{fig:framework}
\end{figure*}

\section{Related Work}


\noindent
\textbf{Data Selection and Mixing for LLMs.} Data mixing strategies are
pivotal for optimizing LLM training stability and generalization~\cite{gpt3,llama,internlm}.
Recent research has introduced adaptive reweighting frameworks. For instance,
\textbf{DoReMi}~\cite{doremi}, \textbf{DoGE}~\cite{fan2023doge}, Aioli~\cite{chen2024aioli}
and \textbf{RegMix}~\cite{regmix} utilize training signals such as gradient
alignment, excess loss or performance regression to dynamically adjust domain
weights. Extending this granularity, \textbf{SampleMix}~\cite{samplemix}
performs evaluation of each individual sample, while \textbf{TikMix}~\cite{tikmix}
dynamically recalibrates mixing weights based on data influence. Furthermore,
\textbf{QuadMix}~\cite{quadmix} introduces a unified objective to assess both
data quality and diversity. Nevertheless, these methods typically treat the underlying
categorization as an exogenous constant. Their effectiveness is
fundamentally upper-bounded by the quality of the initial partition. If the taxonomy
is semantically misaligned or noisy, even these sophisticated mixing
algorithms will struggle to isolate high-utility data features. To address this
limitation, we argue that refined structural granularity is a prerequisite
for effective mixing. We propose a geometry-aware classification scheme that
induces semantically coherent partitions from the latent space, enabling
robust mixing on high-entropy web data.


\noindent
\textbf{Pretraining Data Categorization.} Several recent works have
explicitly addressed the problem of categorizing large-scale pre-training
data, which can be broadly divided into taxonomy-based and unsupervised methods.
\textbf{Taxonomy-based approaches} rely on predefined label systems and assign
categories using supervised classifiers or LLMs. Systems such as
WebOrganizer and TnT-LLM employ LLM-based pipelines to annotate web
documents into manually designed taxonomies~\cite{tnt,weborganizer}. While these
methods yield human-interpretable labels, they suffer from two limitations: (i)
the imposed taxonomies reflect human-defined ontologies rather than the latent
semantic structure learned by the model, leading to potential ontological
misalignment; and (ii) large-scale inference with LLMs incurs substantial computational
cost, limiting scalability. \textbf{Unsupervised categorization} methods
avoid manual labels and instead cluster representations produced by
pretrained encoders. Typical techniques include K-Means or density-based clustering
algorithms such as HDBSCAN~\cite{kmeans,hdbscan}, as adopted in systems like
NVIDIA Climb~\cite{climb}. Although scalable and label-free, these
approaches generally operate in Euclidean space and rely on distance-based
objectives. In high-dimensional embedding spaces, however, distances tend to
concentrate, making Euclidean proximity a weak proxy for semantic similarity.




\section{Methodology}
\label{sec:method}

We introduce \textbf{GEM (Geometric Entropy Mixing)}, a spherical mixture
modeling framework for unsupervised semantic partitioning of web-scale text
embeddings. GEM is built on directional statistics on the unit hypersphere $\mathcal{S}
    ^{d-1}$ and optimizes an entropy-regularized variational objective augmented
with an explicit mixing-balance regularizer to mitigate cluster collapse induced
by embedding anisotropy. Figure~\ref{fig:framework} gives an overview. Below,
we describe (i) the geometric problem setup, (ii) an entropy-aware vMF mixture
formulation with a balance regularizer, (iii) a scalable \emph{MM-based} (minorize--maximize)
inference scheme with provable monotonic ascent, (iv) an interpretable taxonomy
generation pipeline leveraging Geometric Influence Scores (GIS), and (v) a
teacher-student distillation framework for efficient deployment on trillion-token
corpora.

\subsection{Problem Reformulation}
\label{subsec:reformulation}

We consider unsupervised semantic partitioning for a massive corpus of normalized
text embeddings $\mathcal{X}=\{x_{i}\}_{i=1}^{N}\subset \mathbb{R}^{d}$,
where each $x_{i}$ is $\ell_{2}$-normalized and thus lies on the unit
hypersphere: $x_{i}\in \mathcal{S}^{d-1}\coloneqq \{x\in\mathbb{R}^{d}:\|x\|_{2}
    =1\}$. Our goal is to learn a partition $\mathcal{C}=\{C_{1},\dots,C_{K}\}$
such that clusters are distinguishable by \emph{semantic directionality},
yielding a robust semantic basis for downstream data mixing in LLM pre-training.
This formulation assumes that the directional geometry of the text embedder
provides a useful proxy for the downstream LLM's data geometry. We do not require
an exact equivalence between the two spaces; rather, the embedder supplies a
stable semantic coordinate system whose utility is empirically validated by
data-mixing predictability and downstream pre-training results.

\noindent
\textbf{Motivation: concentration on high-dimensional spheres.} A classical concentration
phenomenon suggests that Euclidean proximity becomes less informative in high
dimensions; on $\mathcal{S}^{d-1}$, angles between random directions concentrate
near $\pi/2$.
\begin{lemma}[Concentration on hyperspheres
            {\citep{ledoux2001concentration}}]
    \label{lem:sphere_concentration} Let $x \sim \mathrm{Unif}(\mathcal{S}^{d-1}
        )$. For any fixed $p \in \mathcal{S}^{d-1}$ and any $\epsilon>0$,
    \begin{equation}
        \mathbb{P}\!\left(\left|\langle x,p\rangle\right|\le \epsilon\right)
        ~\ge~ 1-2\exp\!\left(-\frac{d\epsilon^{2}}{2}\right).
    \end{equation}
\end{lemma}

\textit{Remark.} Lemma~\ref{lem:sphere_concentration} provides an intuition
that, for $d\gg 1$, random directions are nearly orthogonal. While real neural
embeddings are not uniform on $\mathcal{S}^{d-1}$, they often exhibit strong
anisotropy and ``hubness'', making purely Euclidean clustering unstable.
This motivates modeling directional coherence using spherical distributions
whose sufficient statistic is cosine similarity.

\noindent
\textbf{Variational learning objective.} We seek directional parameters $\Theta$
and soft assignments $\Gamma=\{\gamma_{ik}\}$ that fit a spherical mixture model
while explicitly encouraging \emph{balanced} cluster masses to improve the
diversity of induced data mixtures. Let the empirical (soft) cluster mass be
\begin{align}
    \pi_{k}(\Gamma) \;\coloneqq\; \frac{1}{N}\sum_{i=1}^{N}\gamma_{ik}, \quad & \boldsymbol{\pi}(\Gamma)\in\Delta^{K-1}, \notag \\
    \text{where}\quad                                                         & \mathbf{u}\coloneqq \tfrac{1}{K}\mathbf{1}.
\end{align}
We optimize an entropy-regularized variational lower bound (ELBO) augmented with
a mixing-balance regularizer:
\begin{equation}
    \label{eq:gem_elbo_objective}
    \begin{split}
        \max_{\Theta,\Gamma}\; & \underbrace{\sum_{i=1}^N\sum_{k=1}^K \gamma_{ik}\log\!\big(\alpha_k
            f_{ik}(\Theta)\big) + \sum_{i=1}^N H(\gamma_i)}_{\text{Geometric
        Fidelity (ELBO)}}                                                                                           \\
                               & -\frac{\lambda}{2}\underbrace{\big\|\boldsymbol{\pi}(\Gamma)-\mathbf{u}\big\|_2^2}
        _{\text{Mixing-balance}}, \qquad \lambda>0.
    \end{split}
\end{equation}
where $H(\gamma_{i})\coloneqq -\sum_{k=1}^{K}\gamma_{ik}\log\gamma_{ik}$ and
$f_{ik}(\Theta)\coloneqq f_{\mathrm{vMF}}(x_{i}\mid\mu_{k},\kappa_{k})$ (defined
in Sec.~\ref{subsec:mixture}). Eq.~\eqref{eq:gem_elbo_objective} is a tight
lower bound of the marginal log-likelihood
$\sum_{i=1}^{N}\log\sum_{k=1}^{K}\alpha_{k}f_{ik}(\Theta)$, and becomes exact
when $\gamma_{i}$ equals the posterior responsibilities. The balance term
penalizes deviation from uniform usage of clusters, mitigating degenerate solutions
and stabilizing partitions under anisotropic embedding distributions.

\subsection{Entropy-Aware Directional Mixture Modeling}
\label{subsec:mixture}

\noindent
\textbf{von Mises--Fisher (vMF) components.} We instantiate the directional likelihood
via a mixture of von Mises--Fisher (movMF) distributions~\citep{banerjee2005clustering},
the canonical directional family on $\mathcal{S}^{d-1}$. A component $k$ has
mean direction $\mu_{k}\in\mathcal{S}^{d-1}$ and concentration $\kappa_{k}\ge
    0$:
\begin{equation}
    \begin{aligned}
        f_{\mathrm{vMF}}(x \mid \mu_{k}, \kappa_{k}) & = C_{d}(\kappa_{k}) \exp(\kappa_{k}\mu_{k}^{\top}x),         \\
        C_{d}(\kappa_{k})                            & = \frac{\kappa_k^{d/2-1}}{(2\pi)^{d/2} I_{d/2-1}(\kappa_k)},
    \end{aligned}
\end{equation}
where $I_{\nu}(\cdot)$ is the modified Bessel function of the first kind. The
sufficient statistic $\mu_{k}^{\top}x$ is exactly cosine similarity,
aligning the generative likelihood with modern embedding metrics.

\noindent
\textbf{Decoupling the generative prior from the empirical mass.} To avoid the
``rich-get-richer'' feedback in standard EM due to learned mixture weights,
we fix the generative prior $\alpha_{k}\equiv 1/K$ for all $k$. The resulting
marginal likelihood is
\begin{equation}
    \label{eq:likelihood_uniform_prior}P(x_{i}\mid\Theta) \;=\; \sum_{k=1}^{K}
    \alpha_{k}f_{\mathrm{vMF}}(x_{i}\mid \mu_{k},\kappa_{k}), \qquad \alpha_{k}
    =\frac{1}{K}.
\end{equation}
Importantly, the balance regularizer in Eq.~\eqref{eq:gem_elbo_objective}
acts on the empirical assignment mass $\boldsymbol{\pi}(\Gamma)$, not on the
generative prior $\boldsymbol{\alpha}$.

\begin{proposition}[Concavity, smoothness, and gradient form of the mixing-balance
        regularizer]
    \label{prop:regularizer_properties} Let $\mathbf{u}\coloneqq \tfrac{1}{K}
        \mathbf{1}$ and define
    \begin{equation}
        \label{eq:regularizer_def}R(\boldsymbol{\pi}) \;\coloneqq\; -\frac{\lambda}{2}
        \big\|\boldsymbol{\pi}-\mathbf{u}\big\|_{2}^{2}, \qquad \lambda>0,
    \end{equation}
    for $\boldsymbol{\pi}\in\Delta^{K-1}$. Then $R$ is (i) concave on
    $\mathbb{R}^{K}$ (hence also on $\Delta^{K-1}$), (ii) differentiable with
    gradient
    \begin{equation}
        \label{eq:regularizer_grad}\nabla_{\pi_k}R(\boldsymbol{\pi}) \;=\; -\lambda
        (\pi_{k}-u_{k}), \qquad u_{k}=\tfrac{1}{K},
    \end{equation}
    and (iii) $\lambda$-smooth with respect to the Euclidean norm, i.e., its
    gradient is $\lambda$-Lipschitz:
    \begin{equation}
        \label{eq:regularizer_smooth}\big\|\nabla R(\boldsymbol{\pi})-\nabla
        R(\boldsymbol{\pi}')\big\|_{2}\;\le\; \lambda\big\|\boldsymbol{\pi}-
        \boldsymbol{\pi}'\big\|_{2}, \qquad \forall\,\boldsymbol{\pi},\boldsymbol
        {\pi}'\in\mathbb{R}^{K}.
    \end{equation}
\end{proposition}

\begin{proof}
    Expanding Eq.~\eqref{eq:regularizer_def} gives $R(\boldsymbol{\pi}) = -\tfrac
        {\lambda}{2}\sum_{k=1}^{K}(\pi_{k}-u_{k})^{2}$. Differentiating w.r.t.\ $\pi
        _{k}$ yields Eq.~\eqref{eq:regularizer_grad}. Moreover, $R$ is a negative
    quadratic function with constant Hessian $\nabla^{2}R(\boldsymbol{\pi}) =
        -\lambda I_{K}\preceq 0$, which implies concavity on $\mathbb{R}^{K}$. Finally,
    $\nabla R(\boldsymbol{\pi})=-\lambda(\boldsymbol{\pi}-\mathbf{u})$ is an
    affine map, hence $\|\nabla R(\boldsymbol{\pi})-\nabla R(\boldsymbol{\pi}
        ')\|_{2}= \lambda\|\boldsymbol{\pi}-\boldsymbol{\pi}'\|_{2}$,
    establishing $\lambda$-smoothness.
\end{proof}

\subsection{MM-Based Inference and Parameter Updates}
\label{subsec:mm_em}

Directly maximizing Eq.~\eqref{eq:gem_elbo_objective} is nontrivial because
$R(\boldsymbol{\pi}(\Gamma))$ couples all samples through the global mass
$\boldsymbol{\pi}(\Gamma)$. We therefore derive a scalable minorize--maximize
(MM) update that yields a provable monotonic ascent of the objective.

\noindent
\textbf{E-step: an MM surrogate with guaranteed ascent.} Fix $\Theta=\Theta^{(t)}$
and denote $\boldsymbol{\pi}^{(t)}\coloneqq \boldsymbol{\pi}(\Gamma^{(t)})$.
Since $R$ is concave and $\lambda$-smooth, it admits the following global quadratic
minorizer: for any $\boldsymbol{\pi},\boldsymbol{\pi}'\in\Delta^{K-1}$,
\begin{equation}
    \label{eq:concave_smooth_minorizer}R(\boldsymbol{\pi}) \;\ge\; R(\boldsymbol
    {\pi}') + \big\langle \nabla R(\boldsymbol{\pi}'),\, \boldsymbol{\pi}-\boldsymbol
    {\pi}'\big\rangle -\frac{\lambda}{2}\big\|\boldsymbol{\pi}-\boldsymbol{\pi}
    '\big\|_{2}^{2}.
\end{equation}
Applying Eq.~\eqref{eq:concave_smooth_minorizer} with $\boldsymbol{\pi}'=\boldsymbol
    {\pi}^{(t)}$ and $\boldsymbol{\pi}=\boldsymbol{\pi}(\Gamma)$ yields a tight lower
bound of the regularizer at $\Gamma^{(t)}$. Consequently, we define the MM surrogate
(a global lower bound that matches the objective at $\Gamma^{(t)}$):
\begin{multline}
    \label{eq:mm_surrogate}\widetilde{\mathcal{F}}_{t}(\Gamma) \;\coloneqq\;
    \sum_{i=1}^{N}\sum_{k=1}^{K}\gamma_{ik}\log\!\big(\alpha_{k}f_{ik}(\Theta
    ^{(t)})\big) + \sum_{i=1}^{N}H(\gamma_{i}) \\
    + R(\boldsymbol{\pi}^{(t)}) + \big\langle \nabla R(\boldsymbol{\pi}^{(t)}
    ),\, \boldsymbol{\pi}(\Gamma)-\boldsymbol{\pi}^{(t)}\big\rangle - \frac{\lambda}{2}
    \big\|\boldsymbol{\pi}(\Gamma)-\boldsymbol{\pi}^{(t)}\big\|_{2}^{2}.
\end{multline}
We then update assignments by maximizing this surrogate:
\begin{equation}
    \label{eq:mm_estep}\Gamma^{(t+1)}\;\in\; \arg\max_{\Gamma:\;\gamma_i\in\Delta^{K-1}}
    \;\widetilde{\mathcal{F}}_{t}(\Gamma).
\end{equation}

\textit{Monotonicity guarantee.} By construction,
$\widetilde{\mathcal{F}}_{t}(\Gamma)\le \mathcal{F}(\Theta^{(t)},\Gamma)$
for all $\Gamma$, and $\widetilde{\mathcal{F}}_{t}(\Gamma^{(t)})=\mathcal{F}(
    \Theta^{(t)},\Gamma^{(t)})$. Therefore, maximizing the surrogate yields a monotonic
ascent:
\begin{equation}
    \label{eq:mm_monotone}\mathcal{F}(\Theta^{(t)},\Gamma^{(t+1)}) \;\ge\; \mathcal{F}
    (\Theta^{(t)},\Gamma^{(t)}).
\end{equation}
A complete monotonic convergence statement for the full GEM iterations is
provided in Appendix~\ref{app:monotone_gem}.

\noindent
\textbf{Practical solver for the MM E-step.} The surrogate in Eq.~\eqref{eq:mm_surrogate}
is a concave objective over $\{\gamma_{i}\in\Delta^{K-1}\}$ (entropy plus a
concave quadratic in $\boldsymbol{\pi}(\Gamma)$), and can be efficiently
optimized with a few steps of projected (mirror) ascent. In practice, we
solve Eq.~\eqref{eq:mm_estep} approximately to a prescribed tolerance; any update
that increases $\widetilde{\mathcal{F}}_{t}$ preserves the ascent property in
Eq.~\eqref{eq:mm_monotone}.

\noindent
\textbf{M-step: closed-form updates on the sphere.} Given $\Gamma^{(t+1)}$, update
the empirical mass:
\begin{equation}
    \pi_{k}^{(t+1)}= \frac{1}{N}\sum_{i=1}^{N}\gamma_{ik}^{(t+1)}.
\end{equation}
For mean directions, the vMF maximum-likelihood update normalizes the weighted
resultant vector:
\begin{equation}
    \label{eq:mu_update_new}r_{k}^{(t+1)}= \sum_{i=1}^{N}\gamma_{ik}^{(t+1)}x
    _{i}, \qquad \mu_{k}^{(t+1)}= \frac{r_{k}^{(t+1)}}{\|r_{k}^{(t+1)}\|_{2}+\varepsilon}
    ,
\end{equation}
where $\varepsilon>0$ is a small constant to avoid numerical issues (e.g., near-empty
clusters with $\|r_{k}\|\approx 0$). For concentration parameters, let
\begin{equation}
    \bar{R}_{k}^{(t+1)}\coloneqq \frac{\|r_{k}^{(t+1)}\|_{2}}{\sum_{i=1}^{N}\gamma_{ik}^{(t+1)}}
    \in [0,1),
\end{equation}
and estimate $\kappa_{k}$ using the standard high-dimensional approximation:
\begin{equation}
    \label{eq:kappa_update_new}\kappa_{k}^{(t+1)}\approx \frac{\bar{R}_{k}^{(t+1)}d
        - (\bar{R}_{k}^{(t+1)})^{3}}{1-(\bar{R}_{k}^{(t+1)})^{2}}.
\end{equation}

\noindent
\textbf{Summary.} GEM combines (i) a spherical vMF mixture with fixed uniform
generative prior $\boldsymbol{\alpha}$ and (ii) a mixing-balance regularizer
on the empirical mass $\boldsymbol{\pi}(\Gamma)$, optimized via an MM-based
E-step with guaranteed monotonic ascent (Eqs.~\eqref{eq:concave_smooth_minorizer}--\eqref{eq:mm_monotone})
and closed-form M-step updates (Eqs.~\eqref{eq:mu_update_new}--\eqref{eq:kappa_update_new}).
This yields semantically coherent yet diverse clusters suitable for robust data
mixing in LLM pre-training. The pseudocode of the overall process is
provided in Appendix~\ref{sec:pseudocode_gem}.

\begin{figure}[t]
    \centering
    \includegraphics[width=0.8\linewidth]{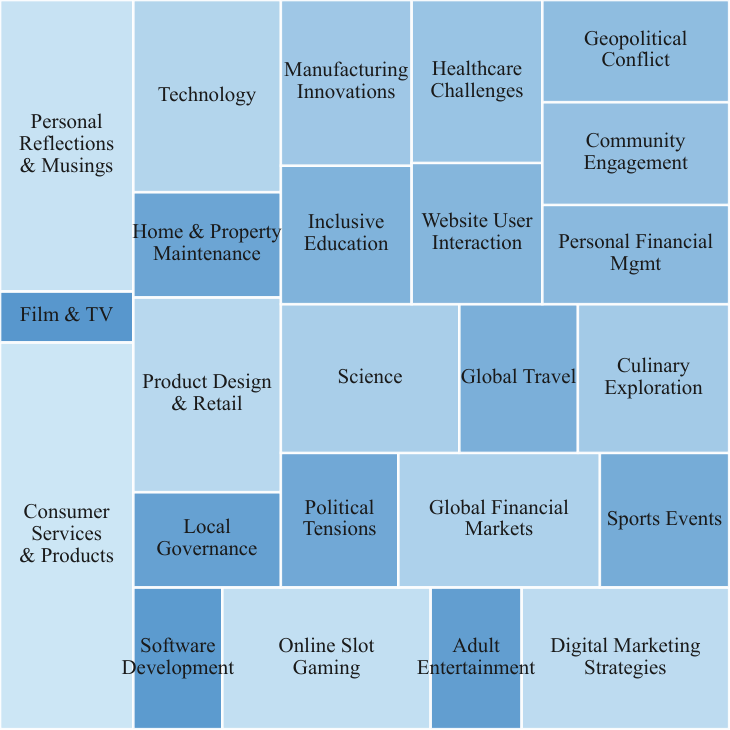}
    \caption{Visualization of the discovered latent semantic taxonomy. The area
        of each rectangle is proportional to the square root of the number of tokens
        in that domain within the pre-training corpus.}
    \label{fig:domain_treemap}
\end{figure}

\subsection{Interpretable Taxonomy Generation}
To transform the unsupervised partition into an interpretable taxonomy for human-understandable
data mixing, we employ a Geometric Influence Score (GIS) to select
representative samples from each cluster, which are then summarized by LLMs to
generate semantic labels. The generated taxonomy is presented in Figure~\ref{fig:domain_treemap}.
The detailed formulation of GIS and the taxonomy generation process are provided
in Appendix~\ref{appendix:gis_sampling}.

\subsection{Scalable Deployment of GEM for Data Mixing}

Direct application of GEM to trillion-token corpora is computationally
prohibitive due to the cost of iterative EM optimization. We bridge this gap
using a Teacher-Student distillation framework. In this framework, GEM serves
as the "Teacher," operating on a representative seed corpus to discover
latent semantic structures. We then utilize the proposed Geometric Influence
Score (GIS) to curate a high-confidence, balanced training set from these
clusters. This dataset supervises a ``Student'' model, a lightweight linear
classifier which approximates the GEM-induced partition function. This design
aligns with established industry practices for processing web-scale data,
where linear classifiers are preferred for their strict latency constraints~\cite{wenzek2020ccnet,soldaini2024dolma}.
By distilling the hyperspherical geometry into a fast inference model, we
enable semantic partitioning over the full pre-training corpus with
negligible computational overhead. Detailed procedures for the clustering,
pseudo-labeling, and distillation phases are provided in Appendix~\ref{app:scalable_deployment}.

\begin{table*}
    [t]
    \centering
    \caption{Main results on downstream tasks. We compare our \textbf{GEM}
        sampling strategy against baseline methods (K-Means, Spherical K-Means, TOS) and existing semantic
        organizers (WebOrganizer) across three different pre-training frameworks.
        The best results within each group are highlighted in \textbf{bold}, and
        the second best results are highlighted by \underline{underline}.}

    \label{tab:main_results}
    \setlength{\tabcolsep}{8pt} 
    \renewcommand{\arraystretch}{1.1} 
    \begin{tabularx}
        {\textwidth}{@{}X|c>{\centering\arraybackslash}X>{\centering\arraybackslash}Xc}
        \toprule \textbf{Model Variation} & \textbf{Science QA}           & \textbf{Commonsense
        Reasoning}                        & \textbf{Logic \& Linguistics} & \textbf{Average}                                            \\
        \midrule \multicolumn{5}{c}{Under DoReMi}                                                                                       \\
        \midrule K-Means                  & 32.18                         & 34.21               & 53.43             & 39.94             \\
        Spherical K-Means                 & 34.62                         & \underline{38.97}   & 54.72             & 42.77             \\
        TOS                               & 32.49                         & 34.29               & 53.98             & 40.25             \\
        WebOrganizer Topic                & \underline{34.68}             & 38.26               & \underline{55.35} &
        42.76                                                                                                                           \\
        WebOrganizer Format               & 34.44                         & 38.73               & 55.19             & \underline{42.79}
        \\
        \textbf{GEM (Ours)}               & \textbf{34.79}                & \textbf{39.96}      & \textbf{57.11}
                                          & \textbf{43.95}                                                                              \\

        \midrule \multicolumn{5}{c}{Under Perf}                                                                                         \\
        \midrule K-Means                  & 32.48                         & 35.11               & 55.52             & 41.04             \\
        Spherical K-Means                 & 34.56                         & \underline{39.95}   & 53.85             & 42.78             \\
        TOS                               & 32.54                         & 34.51               & 53.29             & 40.11             \\
        WebOrganizer Topic                & 35.05                         & 39.73               & 57.90             & 44.23             \\
        WebOrganizer Format               & \underline{35.06}             & 39.73               & \underline{57.97}
                                          & \underline{44.25}                                                                           \\
        \textbf{GEM (Ours)}               & \textbf{35.96}                & \textbf{40.43}      & \textbf{57.98}
                                          & \textbf{44.79}                                                                              \\

        \midrule \multicolumn{5}{c}{Under RegMix}                                                                                       \\
        \midrule K-Means                  & 31.63                         & \underline{34.86}   & \textbf{55.67}    & 40.72
        \\
        Spherical K-Means                 & 32.58                         & \underline{34.86}   & 54.47             & 40.63             \\
        TOS                               & 32.18                         & 34.35               & 52.40             & 39.64             \\
        WebOrganizer Topic                & 33.90                         & 33.83               & 52.50             & 40.08             \\
        WebOrganizer Format               & \textbf{34.12}                & 33.94               & 54.26             & \underline{40.77}
        \\
        \textbf{GEM (Ours)}               & \underline{34.07}             & \textbf{35.30}      & \underline{54.97}
                                          & \textbf{41.45}                                                                              \\
        \bottomrule
    \end{tabularx}%
\end{table*}

\section{Experiments}
\label{sec:experiments}

In this section, we evaluate the effectiveness of our proposed taxonomy in
the context of large-scale data mixing. We aim to demonstrate that our classification
scheme provides a more granular and semantically coherent partitioning of web
data, leading to superior downstream model performance when combined with
various data selection and reweighting algorithms.

\subsection{Experimental Settings}
\label{subsec:settings}

\noindent
\textbf{Datasets and Models.} To systematically investigate the
effectiveness of data classification for data mixing on unlabeled web-scale
corpora, we construct our pre-training dataset from raw CommonCrawl (CC)~\cite{raffel2020exploring}
data. We apply a rigorous cleaning and filtering pipeline that closely
follows the protocols established in RefinedWeb~\cite{refinedweb}, ensuring high-quality
and noise-reduced training data. We adopt a LLaMA-style Transformer
architecture~\cite{llama} with 1.1B parameters. The model configuration is
trained with a fixed compute budget of 25 billion tokens to ensure fair and
controlled comparisons across different data mixing strategies. Details of the
pre-training hyperparameters and experimental infrastructure are provided in
Appendix~\ref{pretraining_details}.

\noindent
\textbf{Categorization Baselines.} We compare our proposed taxonomy against
several established data organization paradigms: (1) \textbf{TOS}~\cite{tos},
which employs a human-curated hierarchical semantic structure for domain partitioning;
(2) \textbf{WebOrganizer (Topic)} and (3) \textbf{WebOrganizer (Format)}, which
leverage LLM-derived topic labels and structural layout features, respectively~\cite{weborganizer};
(4) \textbf{K-means}~\cite{kmeans},
an unsupervised clustering baseline that partitions data based on document-level
embeddings; and (5) \textbf{Spherical K-Means}, which replaces Euclidean distance
with cosine geometry while retaining hard assignments. Details for implementation
and training are provided in Appendix~\ref{appendix:data_classification}.

\noindent
\textbf{GEM Hyperparameters.} Unless otherwise noted, we set $K=24$ and
$\lambda=5000$ in the main experiments. The value of $\lambda$ is chosen by
aligning the balance regularizer with the assignment-logit scale induced by
the vMF likelihood, whose learned concentrations are typically around
$\kappa\approx900$. We provide sensitivity analyses for both $\lambda$ and
the seed-corpus size in Appendix~\ref{appendix:robustness_sensitivity}.

\noindent
\textbf{Data Mixing Strategies.} To evaluate the downstream utility of these
taxonomies, we integrate them with the following mixing algorithms:
(1) \textbf{Perf}, which utilizes a sensitivity-driven approach by upsampling
high-impact categories identified through preliminary performance gains~\cite{tos};
(2) \textbf{RegMix}, which formulates data mixing as a regression task to
predict model performance across different distribution vectors~\cite{regmix};
and (3) \textbf{DoReMi}, which exploits a distributionally robust optimization
objective to minimize the maximum excess loss via a proxy model~\cite{doremi}.
We provide the detailed search space and computational budgets for ratio
optimization in Appendix \ref{appendix:implementation}.

\noindent
\textbf{Downstream Task Evaluation.} We evaluate the zero-shot generalization
capabilities of our pre-trained models using the \textbf{OLMES} ~\cite{gu2025olmes}
framework across nine benchmark datasets. To provide a granular analysis of
model performance, we categorize these benchmarks into three core dimensions:(1)
\textbf{Science QA}: ARC-Challenge~\cite{arc}, ARC-Easy~\cite{arc}, SciQ~\cite{sciq},
and OpenBookQA~\cite{openbookqa}. (2) \textbf{Commonsense Reasoning}: HellaSwag~\cite{hellaswag},
PIQA~\cite{Piqa}, and CommonsenseQA~\cite{commonsenseqa}. (3) \textbf{Logic
    \& Linguistics}: WinoGrande~\cite{winogrande} and COPA~\cite{copa}. By adhering
to the standardized evaluation matrix in OLMES~\cite{gu2025olmes}, we ensure
consistency in prompt templates and scoring metrics.

\subsection{Main Results}

We evaluate the effectiveness of our proposed method by integrating the GEM-based
taxonomy into three distinct data mixing frameworks: DoReMi, Perf, and
RegMix. Table~\ref{tab:main_results} presents the comparative performance across
three aggregated capabilities and the overall average.

\noindent
\textbf{Downstream Performance Analysis.} The empirical results presented in
Table~\ref{tab:main_results} substantiate the effectiveness of our proposed
GEM strategy, which consistently achieves superior performance across diverse
data mixing frameworks. Specifically, under the \textbf{DoReMi} framework, GEM
exhibits the most pronounced improvements, surpassing the strongest baseline,
WebOrganizer, by substantial margins of 1.23\% and 1.76\% points on
Commonsense Reasoning and Logic \& Linguistics, respectively. This trend
persists within the \textbf{Perf} framework, where GEM establishes best
results in Science QA and Commonsense Reasoning, while maintaining a competitive
edge in Logic \& Linguistics by attaining a score of 57.98\% compared to the
57.97\% achieved by the runner-up. In the \textbf{RegMix} setting, although
the K-Means baseline holds a marginal advantage in Logic \& Linguistics with
a performance of 55.67\%, GEM remains dominant in Science QA and Commonsense
Reasoning, significantly outperforming other semantic-aware organizers.
Collectively, these results underscore the robustness of GEM in selecting high-value
training samples that enhance model reasoning capabilities independent of the
underlying mixing strategy.

\begin{figure}[t]
    \centering
    \includegraphics[width=\linewidth]{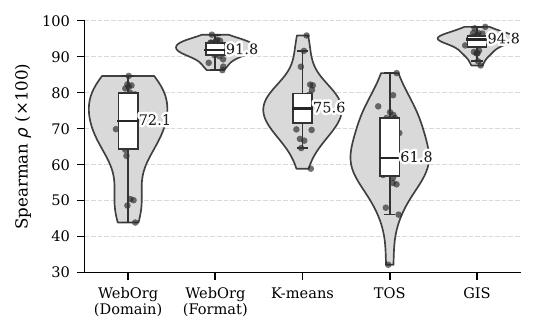}
    \caption{Downstream loss predictability across taxonomies. Violin plots show
        per-subcategory Spearman correlation ($\rho$) between RegMix and ground-truth
        validation loss, averaged over 10 splits. Dots: subcategories; Boxes:
        Median/IQR; Whiskers: $1.5\times$ IQR.}

    \label{fig:regmix_spearman_violin}
\end{figure}


\noindent
\textbf{Evaluating Taxonomy Quality via Data Mixing Predictability.}
\label{subsec:tax_predictability} A robust taxonomy for data mixing must
yield well-conditioned mixing coordinates, where small adjustments to mixture
weights induce consistent and predictable shifts in validation loss. We
quantify this property using \textbf{RegMix}~\cite{regmix} as a \emph{predictability
    probe}. Specifically, we measure the Spearman rank correlation ($\rho$)
between the ground-truth validation loss and the loss predicted by RegMix on
held-out mixture vectors (implementation details in Appendix~\ref{mixing_predic}).
To reduce variance due to a particular split of mixture vectors, we repeat
the procedure over $10$ independent train/test splits and report, for each sub-taxonomy
unit, the average $\rho$ across splits. As illustrated in Figure~\ref{fig:regmix_spearman_violin},
the choice of taxonomy significantly impacts predictability. Baselines such as
K-Means and WebOrganizer exhibit lower and more dispersed $\rho$ values,
indicating unstable mixing dynamics dominated by redundant or entangled
factors. In contrast, \textbf{GEM} demonstrates superior predictability,
achieving consistently higher $\rho$ with a notably tighter distribution. This
reduction in variance suggests that GEM induces a taxonomy with minimized factor
entanglement and a smoother optimization landscape, thereby enabling more
sample-efficient mixture search and reliable control over data composition
during pre-training.

\begin{figure}[t]
    \centering
    \includegraphics[width=0.95\columnwidth]{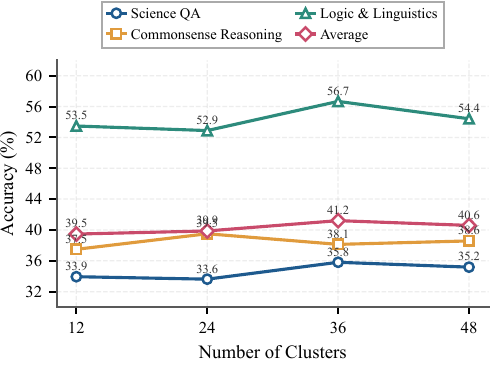}
    \caption{Sensitivity to the number of clusters $K$. Accuracy (\%) on Science
        QA, Commonsense Reasoning, Logic \& Linguistics, and Average as $K$
        varies from 12 to 48 (GEM + Perf.)}
    \label{fig:cluster_ablation}
\end{figure}

\noindent
\textbf{Impact of Cluster Granularity.} In our primary experiments, we set the
number of clusters to $K=24$ to ensure a controlled comparison across all baselines.
To further investigate the optimal granularity for the GIS module, we conduct
a sensitivity analysis by varying $K$ from 12 to 48. As illustrated in Figure~\ref{fig:cluster_ablation},
the model performance manifests a clear dependence on cluster density. Specifically,
we observe a consistent performance trajectory that peaks at $K=36$, achieving
an \textit{Average} score of 41.21\%. This trend suggests that increasing
cluster granularity facilitates the capture of more refined latent semantic
patterns. However, performance begins to plateau or slightly degrade as $K$
reaches 48. This marginal decline is likely attributable to the over-fragmentation
of the embedding space, where excessive partitioning may introduce stochastic
noise or impede the model’s ability to derive robust, generalized
representations.


\begin{figure}[t]
    \centering
    \includegraphics[width=0.95\linewidth]{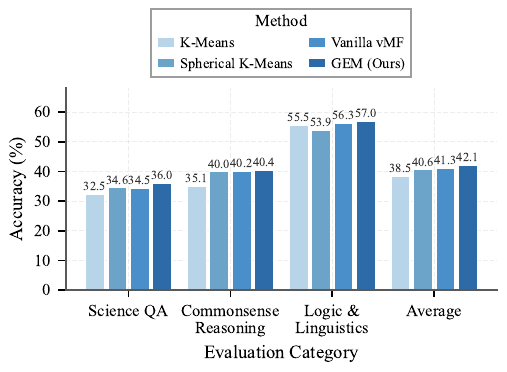}
    \caption{Ablation study on clustering mechanisms. The results highlight the
        improvements gained by shifting from Euclidean to Riemannian geometry and
        by incorporating regularizer.}
    \label{fig:ablation_study}
\end{figure}

\subsection{Ablation Study}
\label{subsec:ablation}

To rigorously assess the contribution of each component within our framework,
we perform an ablation study comparing GEM against three baseline variants: K-Means,
Spherical K-Means (Hyperspherical geometry with hard assignments), and Vanilla
vMF (Riemannian geometry without regularizer). As visualized in Figure~\ref{fig:ablation_study},
the results exhibit a monotonic performance improvement that aligns with the
theoretical fidelity of the clustering objective. The standard K-Means
baseline yields the lowest average accuracy of 38.5\%, corroborating the
limitations of Euclidean metrics in high-dimensional embedding spaces. Transitioning
to Spherical K-Means increases the average accuracy to 40.6\%. Vanilla vMF further
elevates performance to 41.3\% by leveraging probabilistic soft assignments
to capture semantic nuances. Crucially, the proposed GEM framework achieves
the highest average accuracy of 42.1\%, with notable gains in Science QA at
36.0\% and Logic \& Linguistics at 57.0\%. These findings confirm that the gradient-consistent
entropy constraint is essential for mitigating anisotropic cluster collapse,
thereby ensuring a semantically balanced taxonomy that facilitates robust
downstream generalization.

\begin{figure}[t]
    \centering
    \includegraphics[width=0.99\linewidth]{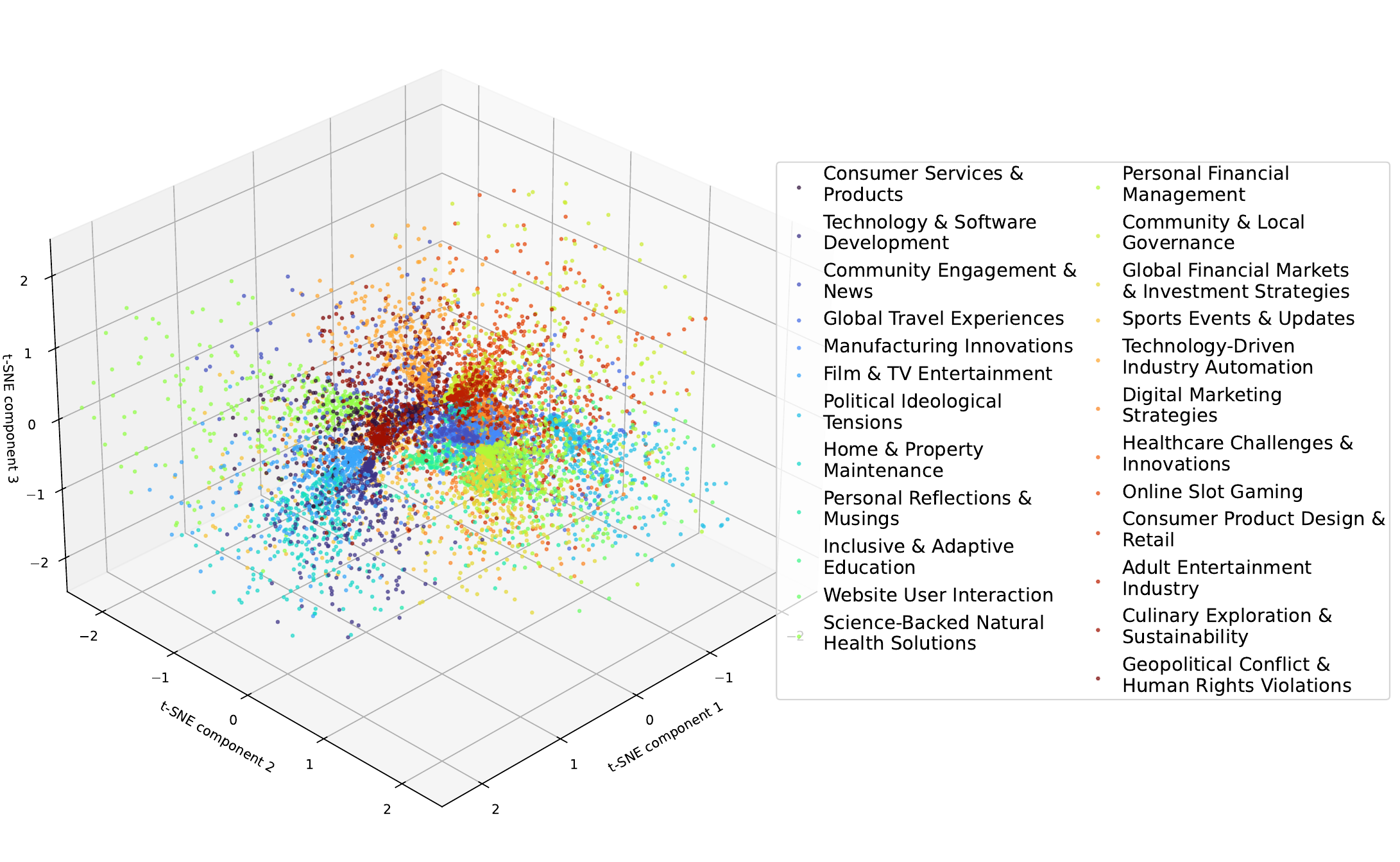}
    \caption{t-SNE visualization of GEM-induced clusters (3D). Colors denote
        the 24 taxonomy topics generated by GIS-based sampling and LLM summarization.}
    \label{fig:tsne_clusters}
\end{figure}

\section{Discussion}
\label{sec:discussion}

A key message of GEM is that ``better'' categorization for LLM data mixing
should be evaluated not only by human interpretability, but by whether the induced
coordinates make mixture search controllable and predictable. Our RegMix-based
predictability probe (Figure~\ref{fig:regmix_spearman_violin}) indicates
that GEM yields a better-conditioned mixing simplex—small weight
perturbations produce more consistent loss orderings—suggesting reduced
factor entanglement relative to Euclidean heuristics. Although GEM clusters are
unlabeled, GIS-guided representative selection makes interpretability a
practical byproduct rather than a fragile post-hoc step: the t-SNE visualization
(Figure~\ref{fig:tsne_clusters}) shows multiple discernible local neighborhoods
alongside overlap in dense regions, reflecting genuine semantic proximity
among discovered topics rather than mere labeling noise.


Several avenues remain for future exploration. First, the current 1.1B-parameter,
25B-token setting should be interpreted as an early-stage scaling indicator;
extending GEM to multi-trillion-token regimes and diverse model architectures
remains necessary to further validate its scaling properties. Second,
investigating GIS as a standalone data curation
primitive can clarify its efficacy in prioritizing high-signal samples beyond
mere partitioning. Third, deployment on continuously updated web corpora should
monitor distribution shifts such as AI-generated content contamination; we
provide a preliminary robustness analysis in Appendix~\ref{appendix:aigc_contamination}.
Finally, exploring the co-optimization of taxonomy discovery
and downstream objectives offers a compelling direction; such bi-level
optimization frameworks could enable a self-improving cycle where the
semantic basis and model performance mutually adapt to enhance task-specific
generalization.

\section{Conclusion}

We presented \textbf{GEM (Geometric Entropy Mixing)}, a geometry-aware framework
for unsupervised data categorization that aligns semantic partitioning with hyperspherical
embedding structures. By combining a vMF mixture model with an explicit
mixing-balance regularizer and an MM-based inference scheme with monotonic
ascent, GEM mitigates anisotropy-induced cluster collapse and yields semantically
coherent yet balanced partitions. To scale to web-scale corpora, we distill GEM’s
decisions into a lightweight classifier and employ GIS-guided sampling for interpretable
taxonomies. Extensive experiments with 1.1B-parameter models show that GEM consistently
improves downstream performance across multiple mixing frameworks. GEM
offers a principled foundation for future work on jointly optimizing
taxonomy discovery and data mixture learning.


\section*{Impact Statement}

This paper presents work whose goal is to advance the field of Machine
Learning. There are many potential societal consequences of our work, none
which we feel must be specifically highlighted here.



\bibliography{example_paper}
\bibliographystyle{icml2026}

\newpage
\appendix
\onecolumn

\section{Monotonic Convergence of GEM}
\label{app:monotone_gem}

\noindent
\textbf{Monotonic ascent guarantee.}
\begin{theorem}[Monotonic convergence of GEM]
    \label{thm:gem_monotone} Let $\mathcal{F}(\Theta,\Gamma)$ denote the GEM
    objective in Eq.~\eqref{eq:gem_elbo_objective}. Fix an iteration $t$ and
    construct the MM surrogate $\widetilde{\mathcal{F}}_{t}(\Gamma)$ as in
    Eq.~\eqref{eq:mm_surrogate} using
    $\boldsymbol{\pi}^{(t)}\coloneqq \boldsymbol{\pi}(\Gamma^{(t)})$. Assume
    the E-step returns $\Gamma^{(t+1)}$ such that
    \begin{equation}
        \label{eq:estep_surrogate_increase}\widetilde{\mathcal{F}}_{t}(\Gamma
        ^{(t+1)}) \;\ge\; \widetilde{\mathcal{F}}_{t}(\Gamma^{(t)}),
    \end{equation}
    and the M-step returns $\Theta^{(t+1)}$ such that
    \begin{equation}
        \label{eq:mstep_increase}\mathcal{F}(\Theta^{(t+1)},\Gamma^{(t+1)}) \;
        \ge\; \mathcal{F}(\Theta^{(t)},\Gamma^{(t+1)}).
    \end{equation}
    Then GEM produces a monotone non-decreasing sequence of objective values:
    \begin{equation}
        \label{eq:gem_monotone_chain}\mathcal{F}(\Theta^{(t+1)},\Gamma^{(t+1)}
        ) \;\ge\; \mathcal{F}(\Theta^{(t)},\Gamma^{(t)}), \qquad \forall\,t\ge
        0.
    \end{equation}
    Consequently, if $\mathcal{F}$ is upper bounded on the feasible set, the
    sequence $\{\mathcal{F}(\Theta^{(t)},\Gamma^{(t)})\}_{t\ge 0}$ converges
    to a finite limit.
\end{theorem}

\begin{proof}
    We first show that $\widetilde{\mathcal{F}}_{t}$ is a valid minorizer of
    $\Gamma\mapsto \mathcal{F}(\Theta^{(t)},\Gamma)$. Let
    $R(\boldsymbol{\pi})\coloneqq -\tfrac{\lambda}{2}\|\boldsymbol{\pi}-\mathbf{u}
        \|_{2}^{2}$. By Proposition~\ref{prop:regularizer_properties}, $R$ is
    concave and $\lambda$-smooth, hence it satisfies the global quadratic minorization
    inequality in Eq.~\eqref{eq:concave_smooth_minorizer}, i.e., for all
    $\boldsymbol{\pi}\in\Delta^{K-1}$,
    \[
        R(\boldsymbol{\pi}) \;\ge\; R(\boldsymbol{\pi}^{(t)}) +\big\langle \nabla
        R(\boldsymbol{\pi}^{(t)}),\, \boldsymbol{\pi}-\boldsymbol{\pi}^{(t)}\big
        \rangle -\frac{\lambda}{2}\big\|\boldsymbol{\pi}-\boldsymbol{\pi}^{(t)}
        \big\|_{2}^{2}.
    \]
    Substituting $\boldsymbol{\pi}=\boldsymbol{\pi}(\Gamma)$ yields, for all
    feasible $\Gamma$,
    \begin{equation}
        \label{eq:minorizer_relation}\widetilde{\mathcal{F}}_{t}(\Gamma) \;\le
        \; \mathcal{F}(\Theta^{(t)},\Gamma), \qquad \widetilde{\mathcal{F}}_{t}
        (\Gamma^{(t)}) \;=\; \mathcal{F}(\Theta^{(t)},\Gamma^{(t)}),
    \end{equation}
    where the equality follows because $\boldsymbol{\pi}(\Gamma^{(t)})=\boldsymbol
        {\pi}^{(t)}$ makes Eq.~\eqref{eq:concave_smooth_minorizer} tight.

    Using Eq.~\eqref{eq:minorizer_relation} and the E-step condition \eqref{eq:estep_surrogate_increase},
    we obtain
    \[
        \mathcal{F}(\Theta^{(t)},\Gamma^{(t+1)}) \;\ge\; \widetilde{\mathcal{F}}
        _{t}(\Gamma^{(t+1)}) \;\ge\; \widetilde{\mathcal{F}}_{t}(\Gamma^{(t)}
        ) \;=\; \mathcal{F}(\Theta^{(t)},\Gamma^{(t)}).
    \]
    Finally, applying the M-step condition \eqref{eq:mstep_increase} gives
    \[
        \mathcal{F}(\Theta^{(t+1)},\Gamma^{(t+1)}) \;\ge\; \mathcal{F}(\Theta
        ^{(t)},\Gamma^{(t+1)}) \;\ge\; \mathcal{F}(\Theta^{(t)},\Gamma^{(t)})
        ,
    \]
    which proves Eq.~\eqref{eq:gem_monotone_chain}. If $\mathcal{F}$ is
    upper bounded on the feasible set, any monotone non-decreasing sequence $\{
        \mathcal{F}(\Theta^{(t)},\Gamma^{(t)})\}_{t\ge 0}$ converges to a finite
    limit.
\end{proof}

\section{Interpretability Analysis: GIS and Taxonomy Details}
\label{appendix:gis_sampling}

\subsection{Interpretability Sampling via GIS}
\label{appendix:gis_sampling_details}

Unsupervised clustering yields a partition function
$x_{i}\mapsto \arg\max_{k}\gamma_{ik}$, but the resulting cluster indices are
not directly human-interpretable. To enable domain-level control in
downstream data mixing (e.g., ``increase Math by 10\%''), we convert each
cluster into a semantic label by prompting an LLM with a small set of \emph{representative}
samples from that cluster. The key challenge is that naive random sampling
often selects either boundary points (ambiguous assignments) or
geometrically central but semantically isolated points under anisotropic
embedding distributions. We therefore propose a principled selection criterion,
the \textbf{Geometric Influence Score (GIS)}, that is consistent with the
GEM objective in Sec.~\ref{sec:method}.

\paragraph{Definition.}
Let $\Gamma$ and $\Theta=\{(\mu_{k},\kappa_{k})\}_{k=1}^{K}$ be the learned
GEM assignments and vMF parameters (Sec.~\ref{subsec:mixture}--\ref{subsec:mm_em}).
For a sample $x_{i}$ and cluster $k$, define the \emph{cluster-conditional} GIS
as
\begin{equation}
    \label{eq:gis_score}\mathrm{GIS}_{k}(x_{i}) ~\coloneqq~ \underbrace{\log\!\big(\gamma_{ik}+\varepsilon\big)}
    _{\text{certainty}}~+~ \underbrace{\log f_{\mathrm{vMF}}(x_i \mid \mu_k,\kappa_k)}
    _{\text{directional coherence}}~+~ \underbrace{\beta \log\!\big(\rho_k(x_i)+\varepsilon\big)}
    _{\text{local support}},
\end{equation}
where $\varepsilon>0$ is a small constant for numerical stability and $\beta\ge
    0$ controls the strength of the density term.

Using the vMF likelihood defined in Sec.~\ref{subsec:mixture},
\begin{equation}
    \log f_{\mathrm{vMF}}(x_{i}\mid \mu_{k},\kappa_{k}) ~=~ \log C_{d}(\kappa
    _{k}) + \kappa_{k}\mu_{k}^{\top}x_{i}.
\end{equation}
Since $\log C_{d}(\kappa_{k})$ is independent of $i$ for fixed $k$, ranking
samples \emph{within the same cluster} is equivalent to using $\kappa_{k}\mu_{k}
    ^{\top}x_{i}$; we keep the full $\log f_{\mathrm{vMF}}$ in Eq.~\eqref{eq:gis_score}
to make the probabilistic meaning explicit and to remain fully consistent
with the generative model.

To penalize semantically isolated points, we estimate a cluster-aware local density
by restricting neighbors to the same cluster:
\begin{equation}
    \label{eq:local_density_revised}\rho_{k}(x_{i}) ~\coloneqq~ \frac{1}{M}\sum
    _{x_j \in \mathcal{N}^{(k)}_M(x_i)}x_{i}^{\top}x_{j}, \qquad \mathcal{N}^{(k)}
    _{M}(x_{i})\subseteq \{x_{j}:\arg\max_{\ell}\gamma_{j\ell}=k\},
\end{equation}
where $\mathcal{N}^{(k)}_{M}(x_{i})$ denotes the $M$ nearest neighbors of
$x_{i}$ \emph{within cluster $k$} under cosine similarity (equivalently
Euclidean distance on $\mathcal{S}^{d-1}$). This cluster-restricted density
avoids cross-cluster contamination and directly measures whether $x_{i}$
lies in a densely populated region of the learned semantic direction.

For each cluster $k$, we select representatives by ranking points assigned to
$k$ using $\mathrm{GIS}_{k}(x_{i})$ and taking the top-$S$ samples:
\begin{equation}
    \mathcal{R}_{k}~\coloneqq~ \operatorname{TopS}\Big(\{\mathrm{GIS}_{k}(x_{i}
    )\}_{i:\arg\max_{\ell}\gamma_{i\ell}=k}\Big).
\end{equation}
The set $\mathcal{R}_{k}$ is then used as few-shot context to prompt an LLM
for a concise semantic label, yielding an interpretable taxonomy consistent with
the GEM-induced partition.


\subsection{GIS Taxonomy Details}
\label{gis_taxonomy}
\subsubsection{Taxonomy Description}

Table~\ref{tab:topic_definitions} provides a comprehensive overview of the semantic
taxonomies identified by our GEM framework. Each taxonomy represents a
distinct semantic cluster discovered through our entropy-regularized vMF
mixture modeling approach, enabling fine-grained categorization of the pre-training
corpus for optimal data mixing strategies.

    {\small
        \begin{longtable}{>{\raggedright\arraybackslash}p{0.22\textwidth} >{\raggedright\arraybackslash}p{0.73\textwidth}}
            \caption{Detailed overview of our topic definitions. This table lists the specific topics and their corresponding descriptions to decrease uncertainty and define domain boundaries.}
            \label{tab:topic_definitions}                                                                                                                                                                                                                                                     \\
            \toprule
            \textbf{Taxonomy}                                  & \textbf{Description}                                                                                                                                                                                                         \\
            \midrule
            \endfirsthead
            \multicolumn{2}{c}%
            {{\bfseries \tablename\ \thetable{} -- continued from previous page}}                                                                                                                                                                                                             \\
            \toprule
            \textbf{Taxonomy}                                  & \textbf{Description}                                                                                                                                                                                                         \\
            \midrule
            \endhead
            \midrule
            \multicolumn{2}{r}{{Continued on next page}}                                                                                                                                                                                                                                      \\
            \endfoot
            \bottomrule
            \endlastfoot
            Consumer Services and Products                     & This topic encompasses a range of services and products aimed at improving consumer experiences, emphasizing sustainability, marketing, and quality in various industries.                                                   \\ \addlinespace
            Technology and Software Development                & This topic encompasses the development, functionality, and improvement of software and technology solutions, focusing on web development, software tools, and user experience enhancements.                                  \\ \addlinespace
            Community Engagement and News                      & The topic focuses on the dissemination of local news, events, and initiatives that encourage community participation and awareness.                                                                                          \\ \addlinespace
            Global Travel Experiences                          & This topic encompasses the exploration and enjoyment of diverse travel destinations and cultural experiences worldwide, facilitated by convenient booking and travel arrangements.                                           \\ \addlinespace
            Manufacturing Innovations                          & This topic encompasses the latest advancements and innovations in manufacturing processes and product design across diverse industries, emphasizing quality, efficiency, and technological improvements.                     \\ \addlinespace
            Film and TV Entertainment                          & This topic encompasses the latest developments, releases, and events in the film and television industry, particularly focusing on superhero and action genres, film festivals, and popular franchises.                      \\ \addlinespace
            Political Ideological Tensions                     & The topic encompasses the ongoing political and ideological conflicts in the United States, emphasizing leadership, civil unrest, religious freedom, and conservative principles.                                            \\ \addlinespace
            Home and Property Maintenance                      & This topic encompasses a range of services aimed at maintaining and improving residential and commercial properties, ensuring they are clean, functional, and aesthetically pleasing.                                        \\ \addlinespace
            Personal Reflections and Musings                   & This topic involves sharing personal anecdotes, contemplations, and humorous insights on everyday life and broader philosophical questions, often in a casual and engaging narrative style.                                  \\ \addlinespace
            Inclusive and Adaptive Education                   & This topic focuses on educational strategies and programs designed to promote inclusivity, accessibility, and adaptability to meet the diverse needs of students and educators in evolving academic and career landscapes.   \\ \addlinespace
            Website User Interaction                           & This topic covers the various aspects of user engagement with websites, including privacy settings, account management, and customer support services.                                                                       \\ \addlinespace
            Science-Backed Natural Health Solutions            & This topic focuses on the integration of scientific research and natural ingredients to promote health and wellness, offering solutions for skincare, energy enhancement, and weight management.                             \\ \addlinespace
            Personal Financial Management                      & This topic encompasses strategies and resources for effectively managing personal finances in diverse situations, including moving, property division, home improvement, and budgeting.                                      \\ \addlinespace
            Community and Local Governance                     & The topic encompasses community resilience, local government initiatives, and administrative challenges within county-level governance.                                                                                      \\ \addlinespace
            Global Financial Markets and Investment Strategies & This topic encompasses the analysis of economic trends, investment strategies, and the impact of government policies on global financial markets.                                                                            \\ \addlinespace
            Sports Events and Updates                          & This topic encompasses the latest developments, competitions, and strategic changes in various sports, providing insights into the dynamic world of sports management and events.                                            \\ \addlinespace
            Technology-Driven Industry Automation              & The cluster explores how cutting-edge technologies and automation are transforming industry practices, enhancing operational efficiency, and driving innovation across sectors.                                              \\ \addlinespace
            Digital Marketing Strategies                       & This topic covers methods and tools for optimizing digital marketing to improve business visibility, audience engagement, and overall success.                                                                               \\ \addlinespace
            Healthcare Challenges and Innovations              & This topic encompasses the exploration of healthcare challenges, patient experiences, and innovative approaches to treatment and care delivery.                                                                              \\ \addlinespace
            Online Slot Gaming                                 & Online Slot Gaming refers to the digital version of traditional slot machines, offering players an engaging and convenient way to experience casino excitement with diverse themes and features.                             \\ \addlinespace
            Consumer Product Design and Retail                 & This topic encompasses the design, quality, and retail aspects of various consumer products, focusing on unique craftsmanship and customer shopping experiences.                                                             \\ \addlinespace
            Adult Entertainment Industry                       & This topic encompasses the production, distribution, and consumption of adult content, including live performances, online platforms, and the economic aspects of the adult entertainment sector.                            \\ \addlinespace
            Culinary Exploration and Sustainability            & This topic encompasses the art of cooking, the enjoyment of diverse flavors, and the significance of sustainable food practices.                                                                                             \\ \addlinespace
            Geopolitical Conflict and Human Rights Violations  & This topic encompasses the examination of violence, military actions, and human rights abuses in conflict zones, highlighting the systematic nature of these issues and their impact on civilians and freedom of expression. \\
            \bottomrule
        \end{longtable}
    }

\subsubsection{Prompt for Taxonomy Generation}

\begin{tcolorbox}
    [ enhanced, colback=gray!5, 
        colframe=blue!50!black, 
        title=Prompt for Taxonomy Generation, 
        fonttitle=\bfseries, arc=2mm, 
        drop shadow, 
        left=3mm, right=3mm, top=2mm, bottom=2mm 
    ] \itshape 
    You are an expert data taxonomist. \\
    I will provide you with \{len(indices)\} documents that belong to the
    same semantic cluster. \\
    Your task is to:
    \begin{enumerate}
        \item Summarize the common theme and content of these documents in 2-3
              sentences.

        \item Based on the summary, assign a single, concise, and high-quality
              'Topic Label' (2-5 words) that best describes this cluster.

        \item Describe the topic in a sentence.
    \end{enumerate}

    \textbf{Documents:} \\
    \{docs\_content\} \\

    Please strictly follow the output format: \\
    \textbf{Summary:} {summary content} \\
    \textbf{Topic:} {topic label} \\
    \textbf{Description:} {topic description}
\end{tcolorbox}

\subsection{Quantitative Coherence and GIS Ablations}
\label{appendix:taxonomy_coherence_gis_ablation}

To complement the qualitative taxonomy visualization in Figure~\ref{fig:domain_treemap},
we evaluate semantic consistency using automated topic-coherence metrics. For
each categorization strategy, we extract the top-15 TF-IDF words per cluster
and compute Normalized Pointwise Mutual Information (NPMI) and the $C_{V}$
score over a large sampled corpus. As shown in Table~\ref{tab:topic_coherence},
GEM achieves topic coherence on par with the LLM-annotated WebOrganizer Topic
taxonomy while retaining a fully unsupervised, model-native partitioning process.

\begin{table}[h]
    \centering
    \small
    \caption{Automated topic-coherence evaluation across taxonomies. Higher is better.}
    \label{tab:topic_coherence}
    \setlength{\tabcolsep}{12pt}
    \renewcommand{\arraystretch}{1.1}
    \begin{tabular}{lcc}
        \toprule
        \textbf{Taxonomy Strategy} & \textbf{NPMI} $\uparrow$ & \textbf{$C_{V}$ Score} $\uparrow$ \\
        \midrule
        WebOrganizer Format        & -0.1219                  & 0.4341                            \\
        WebOrganizer Topic         & \textbf{-0.0700}         & \textbf{0.4943}                   \\
        GEM (Ours)                 & \underline{-0.0728}      & \underline{0.4928}                \\
        \bottomrule
    \end{tabular}
\end{table}

We further ablate the representative-sample selection rule used for taxonomy
generation. For each cluster, we select five samples under each heuristic,
prompt an LLM to generate the topic description, and then use the description
as a prompt for labeling a held-out set of 720 samples. Table~\ref{tab:gis_ablation}
shows that GIS yields the highest labeling accuracy, indicating that combining
assignment certainty, directional coherence, and local support produces more
semantically robust prototypes than simpler alternatives.

\begin{table}[h]
    \centering
    \small
    \caption{Ablation study on prototype selection heuristics for taxonomy generation.}
    \label{tab:gis_ablation}
    \setlength{\tabcolsep}{12pt}
    \renewcommand{\arraystretch}{1.1}
    \begin{tabular}{llc}
        \toprule
        \textbf{Selection Method} & \textbf{Selection Criteria}              & \textbf{Accuracy (\%)} \\
        \midrule
        Random Sampling           & Uniform random                           & 80.41                  \\
        Confidence-only           & Maximize $\log \gamma_{ik}$              & 80.83                  \\
        Center-only               & Maximize $\kappa_{k}\mu_{k}^{\top}x_{i}$ & 81.67                  \\
        GIS (Ours)                & Eq.~\eqref{eq:gis_score}                 & \textbf{84.17}         \\
        \bottomrule
    \end{tabular}
\end{table}

\subsection{Qualitative Comparison with K-Means}
\label{appendix:kmeans_case_analysis}

We also inspect representative cases where GEM and K-Means produce different
cluster assignments. In one case, documents about configuring a vanity URL
shortener and open-source conference sprints were separated by K-Means into
surface-level website-interaction and politics-related clusters, whereas GEM
placed both into a coherent technology and software-development cluster. In
another case, a technical endoscopic-surgery abstract and a sunscreen news
article shared generic medical vocabulary and were grouped together by K-Means;
GEM separated them into healthcare-innovation and natural-health clusters,
respectively. These cases illustrate that GEM's directional probabilistic
modeling can reduce both unnatural fragmentation and inappropriate grouping
caused by superficial lexical overlap.

\section{Details of Scalable Deployment}
\label{app:scalable_deployment}

To scale GEM to trillion-token datasets, we employ a two-phase Teacher-Student
approach. This ensures that the high-quality geometric partitions discovered
by GEM can be applied to web-scale data without incurring the computational cost
of the EM algorithm on the entire corpus.

\noindent
\textbf{Phase 1: Clustering and GIS-Guided Pseudo-Labeling.} We first apply
GEM to a randomly sampled subset
$\mathcal{X}_{\mathrm{seed}}\subset \mathcal{X}$ to obtain the converged
directional parameters $\Theta^{*}= \{(\mu_{k}, \kappa_{k})\}_{k=1}^{K}$. Using
these learned directions as anchors, we retrieve candidate samples from the
larger unlabeled pool. To mitigate topic-frequency skew and ensure the student
model learns from semantically representative prototypes, we enforce cluster-wise
balanced sampling governed by the GIS score (defined in Eq.~\ref{eq:gis_score}).
Specifically, for each cluster $C_{k}$, we select the top-$M$ samples ranked
by their GIS values, yielding a balanced pseudo-labeled dataset:
\begin{equation}
    \mathcal{D}_{\mathrm{train}}= \bigcup_{k=1}^{K}\left\{ (w_{i}, y_{i}= k )
    \;\middle|\; x_{i}\in \mathrm{Top\text{-}}M \bigl( \text{GIS}(x \mid \Theta
        ^{*}) \bigr) \right\},
\end{equation}
where $w_{i}$ denotes the raw text corresponding to embedding $x_{i}$. By prioritizing
samples with high GIS, we ensure that the distilled dataset consists of
candidates with high model certainty, geometric centrality, and dense local
manifolds, effectively filtering out noise and ambiguous samples near decision
boundaries.

\noindent
\textbf{Phase 2: Lightweight Student Classifier Distillation.} We train a
lightweight text-space classifier $f_{\phi}: w \mapsto \Delta^{K-1}$ to
approximate the GEM-induced partition function. In our implementation, we
adopt \textbf{FastText}~\cite{joulin2017bag} due to its constant-time inference
complexity with respect to corpus size and its strong inductive bias toward linear,
directionally separable decision boundaries, which aligns well with the directional
nature of vMF distributions.

The student model is optimized via cross-entropy loss on the curated dataset
$\mathcal{D}_{\mathrm{train}}$:
\begin{equation}
    \mathcal{L}_{\mathrm{student}}(\phi) = - \sum_{(w, y) \in \mathcal{D}_{\mathrm{train}}}
    \log P_{\mathrm{FastText}}(y \mid w; \phi).
\end{equation}
Once trained, this classifier is deployed to process the full-scale pre-training
corpus, assigning a probability distribution over the $K$ latent topics to
every document.

\section{Pre-training Details}
\label{pretraining_details} We conduct pre-training experiments using the
Llama Factory framework on a high-performance compute node equipped with 8$\times$
NVIDIA GeForce RTX 5090 GPUs. The models are optimized using the AdamW
optimizer with a peak learning rate of $4 \times 10^{-4}$, supplemented by a
linear warmup phase of 2,000 steps and a subsequent cosine annealing schedule.
To ensure training stability and throughput, we employ a global batch size of
256, achieved through a configuration of 8 GPUs, a micro-batch size of 16,
and 2 gradient accumulation steps. Regularization is applied via a weight decay
of 0.1, and the maximum sequence length is constrained to a 2,048-token
cutoff. All comparative experiments maintain these standardized hyperparameters
to isolate the empirical impact of the GEM-induced taxonomy on downstream performance.

\section{Details of Implementation for Data Classification Strategies}
\label{appendix:data_classification} To operationalize the semantic partitions
at scale, we train a lightweight fastText~\cite{joulin2017bag} classifier to
approximate the clustering assignment. For a rigorous comparison, we curated
balanced datasets for both the $K$-Means baseline and the proposed GEM
framework by randomly sampling 5,000 documents per cluster, which were subsequently
partitioned into training, validation, and test sets following an 8:1:1 ratio.
For Spherical K-Means, all embeddings are $\ell_2$-normalized and centroids
are renormalized after each update, yielding hard assignments under cosine
similarity with the same $K$ as the corresponding GEM run.
While the classifier trained on $K$-Means annotations achieved a test
accuracy of 72.92\%, the model distilled from GEM partitions attained a
superior accuracy of 75.13\%. This performance advantage indicates that the
partitions generated by GEM on the Riemannian manifold possess higher
intrinsic semantic consistency and clearer separation boundaries compared to
Euclidean clusters, thereby reducing label noise and facilitating more robust
generalization in the student classifier.

\section{Details of Implementation for Data Mixing Strategies}
\label{appendix:implementation}

In this section, we provide the specific implementation details for the data
mixing strategies compared in our main experiments.

\noindent
\textbf{Perf.} Perf utilizes a sensitivity-driven heuristic inspired by task-oriented
sampling to identify and amplify high-value domains. This strategy is similar
to PerfRe introduced by~\cite{tos}. The process begins with a sensitivity
profiling phase, where we measure the contribution of each domain by individually
upsampling it by 30\% in temporary experimental mixtures and evaluating the
resulting proxy model on downstream tasks. Guided by these performance rankings,
we construct the final training distribution through a stratified upsampling
scheme: the top two domains are upsampled by 40\%, the third and fourth ranked
domains are upsampled by 20\%, and the remaining data budget is distributed uniformly
across all other categories.

\noindent
\textbf{REGMIX.} REGMIX formulates data-mixture selection as a regression
problem. It first samples a diverse set of candidate mixtures (e.g., from a
Dirichlet distribution anchored to the base token distribution) and trains proxy
models under these mixtures to obtain a scalar target signal for each candidate
(typically a validation-domain loss). REGMIX then fits a regressor from
mixture weights to the target signal, and uses large-scale mixture simulation
with fast regression inference to rank candidates and derive a robust final training
distribution (e.g., by selecting or averaging the top-ranked predicted
mixtures). In our instantiation, we train \textbf{256} proxy models with
\textbf{1M} parameters, each for \textbf{1B} tokens.

\noindent
\textbf{DoReMi.} DoReMi formulates data-mixture learning as a
distributionally robust (minimax) reweighting problem. It first trains a \emph{reference}
model, and then trains a \emph{proxy} model while \emph{online} updating domain
weights based on per-domain \emph{excess loss} relative to the reference
model, typically using exponentiated-gradient-style updates with smoothing.
The resulting (time-averaged) weights are then used as a fixed data mixture
for training the final model. In our setting, both the reference and proxy
models have \textbf{120M} parameters and are each trained for \textbf{10B}
tokens.

\begin{table}[h]
    \begin{minipage}[t]{0.48\textwidth}
        \centering
        \small
        \caption{Results on Science QA benchmarks.}
        \label{tab:science_qa_details}
        \setlength{\tabcolsep}{5pt}
        \renewcommand{\arraystretch}{1.1}
        \begin{tabular}{l|ccccc}
            \toprule \textbf{Model Variation} & \textbf{ARC-C} & \textbf{ARC-E} & \textbf{SciQ} & \textbf{OBQA} & \textbf{Avg} \\
            \midrule \multicolumn{6}{c}{Under DoReMi}                                                                          \\
            \midrule K-Means                  & 26.62          & 49.30          & 25.60         & 27.20         & 32.18        \\
            TOS                               & 25.34          & 50.00          & 27.00         & 27.60         & 32.49        \\
            WebOrganizer Topic                & 25.43          & 50.20          & 29.10         & 34.00         & 34.68        \\
            WebOrganizer Format               & 25.94          & 50.10          & 28.50         & 33.20         & 34.44        \\
            GEM (Ours)                        & 26.37          & 50.60          & 28.80         & 33.40         & 34.79        \\
            \midrule \multicolumn{6}{c}{Under Perf}                                                                            \\
            \midrule K-Means                  & 24.91          & 48.70          & 25.10         & 31.20         & 32.48        \\
            TOS                               & 25.94          & 48.80          & 28.00         & 27.40         & 32.54        \\
            WebOrganizer Topic                & 25.68          & 50.10          & 29.80         & 34.60         & 35.05        \\
            WebOrganizer Format               & 25.43          & 50.80          & 29.60         & 34.40         & 35.06        \\
            GEM (Ours)                        & 25.94          & 51.90          & 30.00         & 36.00         & 35.96        \\
            \midrule \multicolumn{6}{c}{Under RegMix}                                                                          \\
            \midrule K-Means                  & 26.62          & 47.90          & 25.00         & 27.00         & 31.63        \\
            TOS                               & 26.02          & 48.30          & 26.00         & 28.40         & 32.18        \\
            WebOrganizer Topic                & 26.79          & 49.60          & 26.20         & 33.00         & 33.90        \\
            WebOrganizer Format               & 25.77          & 49.80          & 26.90         & 34.00         & 34.12        \\
            GEM (Ours)                        & 26.28          & 50.20          & 26.40         & 33.40         & 34.07        \\
            \bottomrule
        \end{tabular}
    \end{minipage}
    \hfill
    \begin{minipage}[t]{0.45\textwidth}
        \centering
        \small
        \caption{Results on Commonsense Reasoning benchmarks.}
        \label{tab:commonsense_details}
        \setlength{\tabcolsep}{5pt}
        \renewcommand{\arraystretch}{1.1}
        \begin{tabular}{l|cccc}
            \toprule \textbf{Model Variation} & \textbf{CSQA} & \textbf{Hella} & \textbf{PIQA} & \textbf{Avg} \\
            \midrule \multicolumn{5}{c}{Under DoReMi}                                                         \\
            \midrule K-Means                  & 19.74         & 26.20          & 56.70         & 34.21        \\
            TOS                               & 21.38         & 24.80          & 56.70         & 34.29        \\
            WebOrganizer Topic                & 19.98         & 34.40          & 60.40         & 38.26        \\
            WebOrganizer Format               & 20.80         & 34.80          & 60.60         & 38.73        \\
            GEM (Ours)                        & 20.97         & 37.40          & 61.50         & 39.96        \\
            \midrule \multicolumn{5}{c}{Under Perf}                                                           \\
            \midrule K-Means                  & 22.03         & 29.10          & 54.20         & 35.11        \\
            TOS                               & 21.62         & 28.80          & 53.10         & 34.51        \\
            WebOrganizer Topic                & 20.80         & 35.20          & 63.20         & 39.73        \\
            WebOrganizer Format               & 20.88         & 35.50          & 62.80         & 39.73        \\
            GEM (Ours)                        & 20.88         & 37.60          & 62.80         & 40.43        \\
            \midrule \multicolumn{5}{c}{Under RegMix}                                                         \\
            \midrule K-Means                  & 20.07         & 26.30          & 58.20         & 34.86        \\
            TOS                               & 22.36         & 25.00          & 55.70         & 34.35        \\
            WebOrganizer Topic                & 21.38         & 25.60          & 54.50         & 33.83        \\
            WebOrganizer Format               & 21.21         & 25.30          & 55.30         & 33.94        \\
            GEM (Ours)                        & 21.21         & 27.50          & 57.20         & 35.30        \\
            \bottomrule
        \end{tabular}
    \end{minipage}
\end{table}

\begin{table}[h]
    \centering
    \small
    \caption{Results on Logic \& Linguistics benchmarks.}
    \label{tab:logic_details}
    \setlength{\tabcolsep}{10pt}
    \renewcommand{\arraystretch}{1.1}
    \begin{tabular}{l|ccc}
        \toprule \textbf{Model Variation} & \textbf{Wino} & \textbf{COPA} & \textbf{Avg} \\
        \midrule \multicolumn{4}{c}{Under DoReMi}                                        \\
        \midrule K-Means                  & 48.86         & 58.00         & 53.43        \\
        TOS                               & 49.96         & 58.00         & 53.98        \\
        WebOrganizer Topic                & 51.70         & 59.00         & 55.35        \\
        WebOrganizer Format               & 51.38         & 59.00         & 55.19        \\
        GEM (Ours)                        & 51.22         & 63.00         & 57.11        \\
        \midrule \multicolumn{4}{c}{Under Perf}                                          \\
        \midrule K-Means                  & 50.04         & 61.00         & 55.52        \\
        TOS                               & 49.57         & 57.00         & 53.29        \\
        WebOrganizer Topic                & 49.80         & 66.00         & 57.90        \\
        WebOrganizer Format               & 48.93         & 67.00         & 57.97        \\
        GEM (Ours)                        & 49.96         & 66.00         & 57.98        \\
        \midrule \multicolumn{4}{c}{Under RegMix}                                        \\
        \midrule K-Means                  & 52.33         & 59.00         & 55.67        \\
        TOS                               & 49.80         & 55.00         & 52.40        \\
        WebOrganizer Topic                & 50.99         & 54.00         & 52.50        \\
        WebOrganizer Format               & 50.51         & 58.00         & 54.26        \\
        GEM (Ours)                        & 51.93         & 58.00         & 54.97        \\
        \bottomrule
    \end{tabular}
\end{table}

\section{Details of implementation for Data Mixing Predictability}
\label{mixing_predic} We sample \textbf{256} mixture vectors $\{w_{j}\}_{j=1}^{256}$
from the simplex and, following the standard RegMix protocol, train one
\textbf{proxy model} (1M parameters) per mixture to obtain the corresponding
validation loss $\{L_{j}\}_{j=1}^{256}$. We then fit the RegMix regression model
on pairs $(w_{j}, L_{j})$ using an \textbf{80/20 split} over mixture points
(205 for training, 51 for testing). To assess predictability on the held-out
mixtures, we compute the \textbf{Spearman rank correlation} ($\rho$) between
the ground-truth losses and the RegMix-predicted losses on the test mixtures.
Higher $\rho$ indicates that the regression model preserves the ordering of losses
induced by mixture weights, suggesting that the taxonomy yields less noisy, less
collinear, and more disentangled mixing factors.


\section{Additional Results}
\subsection{Full Downstream Results}
\label{appendix:full_downstream_results}

We report detailed results on all nine benchmarks, categorized into three dimensions:
\textbf{Science QA} (Table~\ref{tab:science_qa_details}), \textbf{Commonsense
    Reasoning} (Table~\ref{tab:commonsense_details}), and \textbf{Logic \&
    Linguistics} (Table~\ref{tab:logic_details}).

\begin{figure}[t]
    \centering
    \includegraphics[width=0.6\linewidth]{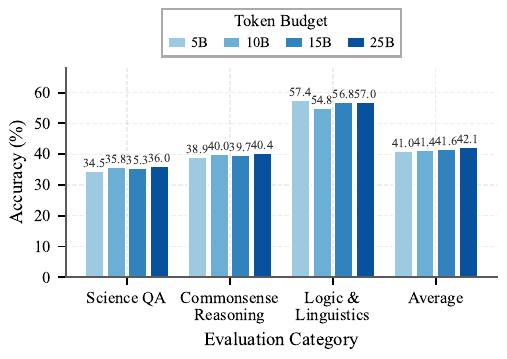}
    \caption{\textbf{Pilot Study on Downscaled Proxies.} The model exhibits a
        consistent and monotonic improvement in \textit{Average} performance (from
        41.05\% to 42.12\%), confirming the stability of our optimization trajectory
        at a reduced scale.}
    \label{fig:pilot_study}
\end{figure}

\subsection{Pilot Study on Downscaled Proxies.}
\

To efficiently validate our training configuration and investigate the early-stage
capability acquisition, we conduct a controlled pilot study using a
downscaled model setup across a range of 5B to 25B tokens. As illustrated in
Figure~\ref{fig:pilot_study}, the model exhibits a consistent and monotonic
improvement in \textit{Average} performance (from 41.05\% to 42.12\%),
confirming the stability of our optimization trajectory at a reduced scale. A
granular analysis reveals diverging learning dynamics across different task categories:
while \textit{Science QA} and \textit{Commonsense Reasoning} show steady
gains—suggesting a continuous integration of world knowledge—the performance
in \textit{Logic \& Linguistics} plateaus relatively early after 5B tokens.
This observation indicates that fundamental linguistic structures are
captured during the initial phase of pre-training, whereas complex reasoning
benefits more significantly from extended token exposure.

\section{Additional Robustness and Sensitivity Analyses}
\label{appendix:robustness_sensitivity}

\subsection{Robustness to AI-Generated Content Contamination}
\label{appendix:aigc_contamination}

As web corpora increasingly contain AI-generated content (AIGC), data curation
methods must remain robust to synthetic regions that may be unusually dense in
embedding space. Such density can amplify the cluster-collapse behavior of hard
Euclidean clustering, pulling otherwise distinct human-written documents into
AIGC-dominated clusters. GEM mitigates this failure mode through the
mixing-balance regularizer on the empirical mass $\boldsymbol{\pi}(\Gamma)$,
which penalizes artificially dominant regions while preserving angular semantic
structure. In our corpus, the RefinedWeb-style filtering pipeline removes most
low-quality or synthetic spam, and a Binoculars-based screening\citep{hans2024spotting} estimates the
proportion of likely AI-generated content to be below 0.5\%. As an additional
stress test, mixing natural web data with Cosmopedia, a synthetic web-style
corpus\citep{benallal2024cosmopedia}, caused K-Means to assign 37.5\% of synthetic data to two massive
clusters, whereas GEM maintained balanced semantic partitions without global
collapse.

\subsection{Sensitivity to the Balance Regularizer}
\label{appendix:lambda_sensitivity}

We study the contribution of the mixing-balance regularizer by varying
$\lambda$ while keeping the remaining experimental setting fixed. Since the
E-step assignment logits are dominated by vMF log-likelihood differences whose
scale is governed by learned concentrations around $\kappa\approx900$, we set
$\lambda=5000$ as the default value in the main experiments to align the
regularizer with this logit scale. Table~\ref{tab:lambda_sensitivity} shows
that vanilla vMF clustering without the balance term already provides a strong
hyperspherical baseline, while the full GEM objective achieves the best average
performance. The results are also stable across a broad range of nonzero
$\lambda$ values, suggesting that GEM does not depend on a finely tuned
regularization coefficient.

\begin{table}[h]
    \centering
    \small
    \caption{Ablation study isolating the contribution and sensitivity of the regularizer coefficient $\lambda$ under GEM + Perf.}
    \label{tab:lambda_sensitivity}
    \setlength{\tabcolsep}{10pt}
    \renewcommand{\arraystretch}{1.1}
    \begin{tabular}{lcccc}
        \toprule
        \textbf{Regularizer Coefficient} & \textbf{Science QA} & \textbf{Commonsense} & \textbf{Logic \& Linguistics} & \textbf{Average} \\
        \midrule
        $\lambda=0$ (Vanilla vMF)        & 34.55               & 40.21                & 56.27                         & 43.68            \\
        $\lambda=1000$                   & 35.19               & 40.21                & 55.70                         & 43.70            \\
        $\lambda=5000$ (Adopted)         & \textbf{35.96}      & \textbf{40.43}       & \textbf{57.98}                & \textbf{44.79}   \\
        $\lambda=10000$                  & 35.81               & 39.58                & 56.68                         & 44.05            \\
        \bottomrule
    \end{tabular}
\end{table}

\subsection{Sensitivity to Seed Corpus Size}
\label{appendix:seed_sensitivity}

The teacher phase runs GEM on a representative seed corpus
$\mathcal{X}_{\mathrm{seed}}$ before distilling the learned partitions into the
student classifier. To evaluate the data efficiency of this stage, we vary the
seed corpus from 1B to 10B tokens. As shown in Table~\ref{tab:seed_sensitivity},
GEM already induces a competitive taxonomy with 1B seed tokens, improves at the
3B-token default, and yields only marginal additional gains at 10B tokens. This
supports the use of a moderate seed corpus for efficient web-scale deployment.

\begin{table}[h]
    \centering
    \small
    \caption{Sensitivity analysis on the size of the seed corpus $\mathcal{X}_{\mathrm{seed}}$ under GEM + Perf.}
    \label{tab:seed_sensitivity}
    \setlength{\tabcolsep}{12pt}
    \renewcommand{\arraystretch}{1.1}
    \begin{tabular}{lcccc}
        \toprule
        \textbf{Seed Corpus Size} & \textbf{Science QA} & \textbf{Commonsense} & \textbf{Logic \& Linguistics} & \textbf{Average} \\
        \midrule
        1B Tokens                 & 35.35               & 39.77                & 57.89                         & 44.37            \\
        3B Tokens (Default)       & \textbf{35.96}      & \textbf{40.42}       & 57.98                         & 44.79            \\
        10B Tokens                & 35.60               & 40.12                & \textbf{58.80}                & \textbf{44.86}   \\
        \bottomrule
    \end{tabular}
\end{table}

\subsection{Learned vMF Concentrations}
\label{appendix:kappa_statistics}

To characterize the geometry learned by GEM, we report summary statistics of
the vMF concentration parameter $\kappa$ across different cluster granularities.
Larger $\kappa$ values correspond to tighter directional concentration. The
consistently high values in Table~\ref{tab:kappa_statistics} indicate that the
learned clusters are compact on the hypersphere, providing empirical support
for modeling the embedding space with directional distributions rather than
ambient Euclidean distances alone.

\begin{table}[h]
    \centering
    \small
    \caption{Statistics of the learned vMF concentration parameter $\kappa$ across different numbers of clusters.}
    \label{tab:kappa_statistics}
    \setlength{\tabcolsep}{10pt}
    \renewcommand{\arraystretch}{1.1}
    \begin{tabular}{lccccc}
        \toprule
        \textbf{Number of Clusters} & \textbf{Mean} & \textbf{Median} & \textbf{Min} & \textbf{Max} & \textbf{Std. Dev.} \\
        \midrule
        $K=12$                      & 881.19        & 884.61          & 854.07       & 904.58       & 14.49              \\
        $K=24$                      & 912.93        & 910.19          & 814.34       & 952.59       & 32.26              \\
        $K=36$                      & 931.79        & 934.83          & 822.89       & 965.92       & 28.48              \\
        $K=48$                      & 942.66        & 945.84          & 830.18       & 978.69       & 30.10              \\
        \bottomrule
    \end{tabular}
\end{table}

\section{Pseudocode of GEM}
\label{sec:pseudocode_gem}

Algorithm~\ref{alg:gem_clustering} summarizes the computational pipeline of
GEM, which maximizes the entropy-regularized variational objective in Eq.~\eqref{eq:gem_elbo_objective}
via an MM-based (minorize--maximize) alternating optimization scheme. The algorithm
starts with spherical initialization, then iteratively performs (i) an
\textit{E-step} that increases the MM surrogate $\widetilde{\mathcal{F}}_{t}$
in Eq.~\eqref{eq:mm_surrogate} (hence guaranteeing monotonic ascent of
$\mathcal{F}$; Theorem~\ref{thm:gem_monotone}), and (ii) an \textit{M-step}
that updates vMF parameters in closed form. After convergence, GEM
optionally selects representative samples using Geometric Influence Scores (GIS)
for taxonomy generation, and distills the learned partition into a lightweight
student classifier for web-scale deployment (details in Appendix).

\begin{algorithm}
    [t]
    \caption{GEM: Geometric Entropy Mixing via MM-Based vMF Clustering}
    \label{alg:gem_clustering}
    \begin{algorithmic}
        [1] \REQUIRE Embeddings $\mathcal{X}=\{x_{i}\}_{i=1}^{N}\subset \mathcal{S}
            ^{d-1}$, number of clusters $K$, balance strength $\lambda>0$,
        maximum iterations $T_{\max}$, tolerance $\varepsilon_{\mathrm{stop}}
            >0$, small constant $\varepsilon>0$ for numerical stability, optional:
        GIS neighborhood size $k_{\mathrm{nn}}$ and top-$M$ representatives
        per cluster. \ENSURE Soft assignments $\Gamma=\{\gamma_{ik}\}$,
        parameters $\Theta=\{(\mu_{k},\kappa_{k})\}_{k=1}^{K}$, (optional) representative
        sets $\{\mathcal{S}_{k}\}_{k=1}^{K}$.


        \vspace{2pt}
        \STATE \textbf{// Initialization} \STATE Initialize $\{\mu_{k}^{(0)}\}
            _{k=1}^{K}$ by spherical $k$-means on $\mathcal{X}$. \STATE
        Initialize $\kappa_{k}^{(0)}\leftarrow \kappa_{\mathrm{init}}\ge 0$
        for all $k$. \STATE Initialize responsibilities $\gamma_{ik}^{(0)}\leftarrow
            \frac{1}{K}$ for all $i,k$. \STATE Set fixed generative prior
        $\alpha_{k}\leftarrow \frac{1}{K}$ for all $k$.

        \vspace{2pt}
        \FOR{$t=0$ to $T_{\max}-1$} \STATE Compute empirical masses $\pi_{k}^{(t)}
            \leftarrow \frac{1}{N}\sum_{i=1}^{N}\gamma_{ik}^{(t)}$ for all $k$.
        \STATE Set
        $\boldsymbol{\pi}^{(t)}\leftarrow (\pi_{1}^{(t)},\dots,\pi_{K}^{(t)})$
        and $\mathbf{u}\leftarrow \tfrac{1}{K}\mathbf{1}$. \STATE Compute $\nabla
            R(\boldsymbol{\pi}^{(t)}) \leftarrow -\lambda(\boldsymbol{\pi}^{(t)}-
            \mathbf{u})$ \COMMENT{Eq.~\eqref{eq:regularizer_grad}}

        \vspace{2pt}
        \STATE \textbf{--- E-step (MM): increase the surrogate $\widetilde{\mathcal{F}}
                _{t}$ ---} \STATE Define $\widetilde{\mathcal{F}}_{t}(\Gamma)$ by Eq.~\eqref{eq:mm_surrogate}
        (using $\Theta^{(t)}$ and $\boldsymbol{\pi}^{(t)}$). \STATE Obtain $\Gamma
            ^{(t+1)}$ such that $\widetilde{\mathcal{F}}_{t}(\Gamma^{(t+1)}) \ge
            \widetilde{\mathcal{F}}_{t}(\Gamma^{(t)})$ \COMMENT{Eq.~\eqref{eq:estep_surrogate_increase}}
        \STATE \hspace{0.7cm} \textit{Implementation note:} maximize $\widetilde
            {\mathcal{F}}_{t}$ over $\{\gamma_{i}\in\Delta^{K-1}\}$ using a few steps
        of projected/mirror ascent; any update satisfying the inequality is
        valid.

        \vspace{2pt}
        \STATE \textbf{--- M-step: closed-form vMF updates ---} \FOR{$k=1$ to $K$}
        \STATE
        $r_{k}^{(t+1)}\leftarrow \sum_{i=1}^{N}\gamma_{ik}^{(t+1)}x_{i}$. \STATE
        $N_{k}^{(t+1)}\leftarrow \sum_{i=1}^{N}\gamma_{ik}^{(t+1)}$. \STATE
        $\mu_{k}^{(t+1)}\leftarrow \frac{r_{k}^{(t+1)}}{\|r_{k}^{(t+1)}\|_{2}+\varepsilon}$
        \COMMENT{Eq.~\eqref{eq:mu_update_new}} \STATE $\bar{R}_{k}^{(t+1)}\leftarrow
            \frac{\|r_{k}^{(t+1)}\|_{2}}{N_{k}^{(t+1)}+\varepsilon}$. \STATE
        $\kappa_{k}^{(t+1)}\leftarrow \frac{\bar{R}_{k}^{(t+1)}d - (\bar{R}_{k}^{(t+1)})^{3}}{1-(\bar{R}_{k}^{(t+1)})^{2}+
                \varepsilon}$
        \COMMENT{Eq.~\eqref{eq:kappa_update_new}} \ENDFOR

        \vspace{2pt}
        \STATE \textbf{--- Stopping criterion ---} \STATE Compute $\Delta \leftarrow
            \big|\mathcal{F}(\Theta^{(t+1)},\Gamma^{(t+1)})-\mathcal{F}(\Theta^{(t)}
            ,\Gamma^{(t)})\big|$ \COMMENT{Eq.~\eqref{eq:gem_elbo_objective}} \IF{$\Delta \le \varepsilon_{\mathrm{stop}}$}
        \STATE \textbf{break} \ENDIF \ENDFOR

        \vspace{2pt}
        \STATE \textbf{// Optional: GIS-based representative selection for
            interpretability (Appendix~\ref{appendix:gis_sampling})} \IF{GIS sampling is enabled}
        \FOR{$k=1$ to $K$} \STATE
        $\mathcal{I}_{k}\leftarrow \{i:\; \arg\max_{j}\gamma_{ij}= k\}$
        \COMMENT{hard assignment for sampling only} \FOR{each $i\in\mathcal{I}_{k}$}
        \STATE Compute a GIS score $\mathrm{GIS}(x_{i}; k)$ using
        $\gamma_{ik}$, $\mu_{k}$, $\kappa_{k}$, and $\rho_{\mathrm{local}}(x_{i}
            )$ (see Appendix~\ref{appendix:gis_sampling}). \ENDFOR \STATE
        $\mathcal{S}_{k}\leftarrow$ top-$M$ samples in $\mathcal{I}_{k}$
        with the largest GIS scores. \ENDFOR \ENDIF

        \vspace{2pt}
        \STATE \textbf{return} $\Gamma^{(t+1)}$, $\Theta^{(t+1)}$, and (if
        enabled) $\{\mathcal{S}_{k}\}_{k=1}^{K}$.
    \end{algorithmic}
\end{algorithm}

\end{document}